\documentclass{article}
\usepackage{times,graphicx,url,amsmath,amssymb,verbatim,lineno}
\usepackage[latin1]{inputenc}

\newcommand{\fracpartial}[2]{\frac{\partial #1}{\partial  #2}}
\def\x{\mathbf{x}}

\def\r{\mathbf{r}}

\def\bbeta{\mathbf{\beta}}

\def\bmu{\mathbf{\mu}}
\def\bsigma{\mathbf{\sigma}}

\newtheorem{theorem}{Theorem}[section]
\newtheorem{lemma}[theorem]{Lemma}
\newtheorem{proposition}[theorem]{Proposition}

\newenvironment{claim}[1][Claim]{\begin{trivlist}
\item[\hskip \labelsep {\bfseries #1}]}{\end{trivlist}}
\newenvironment{proof}[1][Proof]{\begin{trivlist}
\item[\hskip \labelsep {\bfseries #1}]}{\end{trivlist}}
\newenvironment{definition}[1][Definition]{\begin{trivlist}
\item[\hskip \labelsep {\bfseries #1}]}{\end{trivlist}}

\title{Inference with Discriminative Posterior}
\author{Jarkko Saloj\"arvi$^\dagger$ \and Kai Puolam\"aki$^\ddagger$ \and Eerika Savia$^\dagger$ \and Samuel Kaski$^\dagger$ \\
Helsinki Institute for Information Technology\\
Department of Information and Computer Science\\
Helsinki University of Technology\\
P.O. Box 5400, FI-02015 TKK, Finland\\
$^\dagger$ Author belongs to the Finnish Centre of Excellence in\\ Adaptive Informatics Research.\\
$^\ddagger$ Author belongs to the Finnish Centre of Excellence in\\ Algorithmic Data Analysis Research.
}
\date{}

\begin{document}
\maketitle

\begin{abstract}
  We study Bayesian discriminative inference given a model family
  $p(c,\x, \theta)$ that is assumed to contain all our prior
  information but still known to be incorrect. This falls in between
  ``standard'' Bayesian generative modeling and Bayesian regression,
  where the margin $p(\x,\theta)$ is known to be uninformative about
  $p(c|\x,\theta)$.  We give an axiomatic proof that
  \emph{discriminative posterior} is consistent for conditional
  inference; using the discriminative posterior is standard practice
  in classical Bayesian regression, but we show that it is theoretically
  justified for model families of joint densities as well. A practical
  benefit compared to Bayesian regression is that the standard methods
  of handling missing values in generative modeling can be extended
  into discriminative inference, which is useful if the amount of data
  is small. Compared to standard generative modeling, discriminative
  posterior results in better conditional inference if the model
  family is incorrect.  If the model family contains also the true
  model, the discriminative posterior gives the same result as
  standard Bayesian generative modeling. Practical computation is done
  with Markov chain Monte Carlo.
\end{abstract}

\section{Introduction}

Our aim is Bayesian discriminative inference in the case where the
model family $p(c,\x,\theta)$ is known to be incorrect. Here $\x$ is a
data vector and $c$ its class, and the $\theta$ are parameters of the
model family. By discriminative we mean predicting the conditional
distribution $p(c\mid \x)$.

\begin{sloppypar}
  The Bayesian approach of using the posterior of the generative model
  family $p(c,\x,\theta)$ has not been shown to be justified in this
  case, and it is known that it does not always generalize well to new
  data (in case of point estimates, see for example
  \cite{Efron75,Nadas88,Ng02}; in this paper we provide a toy example
  that illustrates the fact for posterior distributions). Therefore
  alternative approaches such as Bayesian regression are applied
  \cite{Gelman03}. It can be argued that the best solution is to
  improve the model family by incorporating more prior knowledge. This
  is not always possible or feasible, however, and simplified models
  are being generally used, often with good results. For example, it
  is often practical to use mixture models even if it is known a
  priori that the data cannot be faithfully described by them (see for
  example \cite{McLachlan00}). There are good reasons for still
  applying Bayesian-style techniques \cite{Hansen99} but the general
  problem of how to best do inference with incorrect model families is
  still open.
\end{sloppypar}

In practice, the usual method for discriminative tasks is Bayesian
regression. It disregards all assumptions about the distribution of
$\x$, and considers $\x$ only as covariates of the model for $c$.
Bayesian regression may give superior results in discriminative
inference, but the omission of a generative model for $\x$ (although
it may be readily available) makes it difficult to handle missing
values in the data. Numerous heuristic methods for imputing missing
values have been suggested, see for example \cite{Little02}, but no
theoretical arguments of their optimality have been presented.  Here
we assume that we are given a generative model family of the full data
($\x$, $c$), and therefore have a generative mechanism readily
available for imputing missing values.

From the generative modeling perspective, Bayesian regression ignores
any information about $c$ supplied by the marginal distribution of
$\x$. This is justified if (i) the covariates are explicitly chosen
when designing the experimental setting and hence are not noisy, or
(ii) there is a separate set of parameters for generating $\x$ on the
one hand and $c$ given $\x$ on the other, and the sets are assumed to
be independent in their prior distribution. In the latter case the
posterior factors out into two parts, and the parameters used for
generating $\x$ are neither needed nor useful in the regression
task. See for instance \cite{Gelman03,Sutton06} for more
details. However, there has been no theoretical justification for
Bayesian regression in the more general setting where the independence
does not hold.

For point estimates of generative models it is well known that
maximizing the joint likelihood and the conditional likelihood give in
general different results. Maximum conditional likelihood gives
asymptotically a better estimate of the conditional likelihood
\cite{Nadas88}, and it can be optimized with
expectation-maximization-type procedures
\cite{Salojarvi05icml,Woodland02}.  In this paper we extend that line
of work to show that the two different approaches, joint and
conditional modeling, result in different posterior distributions
which are asymptotically equal only if the true model is within the
model family. We give an axiomatic justification to the discriminative
posterior, and demonstrate empirically that it works as expected. If
there are no covariates, the discriminative posterior is the same as
the standard posterior of joint density
modeling, that is, ordinary Bayesian inference.\\

To our knowledge the extension from point estimates to a posterior
distribution is new. We are aware of only one suggestion, the
so-called supervised posterior \cite{Grunwald02}, which also has 
empirical support in the sense of maximum a posteriori estimates
\cite{Cerquides05}. The posterior has, however, only been justified
heuristically.  

For the purpose of regression, the discriminative posterior makes it
possible to use more general model structures than standard Bayesian
regression; in essence any generative model family can be used.  In
addition to giving a general justification to Bayesian regression-type
modeling, predictions given a generative model family $p(c,\x,\theta)$
should be better if the whole model is (at least approximately)
correct. The additional benefit is that the use of the full generative
model gives a principled way of handling missing values.  The gained
advantage, compared to using the standard non-discriminative
posterior, is that the predictions should be more accurate assuming
the model family is incorrect.

In this paper, we present the necessary background and definitions in
Section 2.  The discriminative posterior is derived briefly from a set
of five axioms in Section 3; the full proof is included as an
appendix. There is a close resemblance to Cox axioms, and standard
Bayesian inference can indeed be derived also from this set. However,
the new axioms allow also inference in the case where the model
manifold is known to be inadequate. In section 4 we show that
discriminative posterior can be extended in a standard manner
\cite{Little02} to handle data missing at random. In Section 5 we
present some experimental evidence that the discriminative posterior
behaves as expected.

\section{Aims and Definitions}

In this paper, we prove the following two claims; the claims follow
from Theorem~\ref{th:posterior}, discriminative posterior, which is the
main result of this paper.
\begin{claim}[Well-known]
  Given a discriminative model, a model $p(c\mid\x;\theta)$ for the
  conditional density, Bayesian regression results in consistent
  conditional inference.
\end{claim}
\begin{claim}[New]
  Given a joint density model $p(c,\x\mid\theta)$ , discriminative
  posterior results in consistent conditional inference.
\end{claim}
In accordance with \cite{Vapnik95}, we call inference consistent if
the utility is maximized with large data sets. This paper proves both
of the above claims. Notice that although the claim 1 is well known,
it has not been proven, aside from the special case where the priors
for the margin $\x$ and $c \mid \x$ are independent, as discussed in
the introduction and in \cite{Gelman03}.

\subsection{Setup}

Throughout the paper, observations are denoted by $(c,\x)$, and
assumed to be i.i.d. We use $\Theta$ to denote the set of all possible
models that could generate the observations. Models that are 
applicable in practice are restricted to a lower dimensional manifold
$\overline\Theta$ of models, $\overline\Theta\subseteq\Theta$. In
other words, the subspace $\overline\Theta$ defines our {\em model
    family}, in this work denoted by a distribution $p(c,\x,\theta)$
parameterized by $\theta \in \overline\Theta$.

There exists a model in $\Theta$ which describes the ``true'' model,
which has actually generated the observations and is typically
unknown. With slight abuse of notation we denote this model by
$\tilde\theta\in\Theta$, with the understanding that it may be outside
our parametric model family.  In fact, in practice no probabilistic
model is perfectly true and is false to some extent, that is, the data
has usually not been generated by a model in our model family,
$\tilde\theta\notin\overline\Theta$.

The distribution induced in the model parameter space $\overline
\Theta$ after observing the data $D=\left\{(c,\x)\right\}_{i=1}^n$ is
referred to as a \emph{posterior}. By standard posterior we mean the
posterior obtained from Bayes formula using a full joint density
model. In this paper we discuss the discriminative posterior, which is
obtained from axioms 1--5 below.

\subsection{Utility Function and Point Estimates}
\label{sec:point estimates}
In this subsection, we introduce the research problem by
recapitulating the known difference between point estimates of joint
and conditional likelihood. We present the point estimates in terms of
Kullback-Leibler divergences, in a form that allows generalizing from
point estimates to the posterior distribution in section
\ref{sec:axiomatic}.

Bayesian inference can be derived in a decision theoretic framework as
maximization of the expected utility of the decision
maker \cite{Bernardo00}. In general, the choice of the utility function
is subjective. However, several arguments for using log-probability as
utility function can be made, see for example~\cite{Bernardo00}.

When inspecting a full generative model at the limit where the amount
of data is infinite, the joint posterior distribution $p_j(\theta\mid
D)\propto p(\theta)\prod_{(c,\x)\in D}{p(c,\x\mid\theta)}$ becomes a
point solution,\, $p_j(\theta\mid
D)=\delta(\theta-\hat\theta)$.\footnote{Strictly speaking, the
  posterior can also have multiple modes; we will not treat these
  special cases here but they do not restrict the generality.}  An
accurate approximation of the log-likelihood is produced by a utility
function minimizing the approximation error $K_{JOINT}$ between the
point estimate $\hat\theta_{JOINT}$ and the true model $\tilde\theta$
as follows:
\begin{eqnarray}
\hat\theta_{JOINT} &=& \arg\min_{\theta\in\overline\Theta}{K_{JOINT}(\tilde\theta,\theta)} \quad \mbox{where}\nonumber\\
K_{JOINT}(\tilde\theta,\theta) &=&
\sum_c{\int{p(c,\x\mid \tilde\theta)
\log\frac{p(c,\x\mid \tilde\theta)}{p(c,\x\mid \theta)}\;d\x}} ,
\label{eq:map}
\end{eqnarray}
If the true model is in the model family, that is,
$\tilde\theta\in\overline\Theta$, equation (\ref{eq:map}) can be
minimized to zero and the resulting point estimate is effectively the
MAP solution.  If $\tilde\theta\notin\overline\Theta$ the resulting
point estimate is the best estimate of the true joint distribution
$p(c,\x\mid \tilde\theta)$ with respect to $K_{JOINT}$.

However, the joint estimate may not be optimal if we are interested in
approximating some other quantity than the likelihood.  Consider the
problem of finding the best point estimate $\hat\theta_{COND}$ for the
conditional distribution $p(c\mid \x,\theta)$.  The average
KL-divergence between the true conditional distribution at
$\tilde\theta$ and its estimate at $\theta$ is given by
\begin{equation}
K_{COND}(\tilde\theta,\theta) =
\int{{p(\x\mid \tilde\theta)\sum_{c}{p(c\mid \x,\tilde\theta)
\log\frac{p(c\mid \x,\tilde\theta)}{p(c\mid \x,\theta)}}}\; d\x} \; ,
\label{eq:cond}
\end{equation}
and the best point estimate with respect to $K_{COND}$  is
\begin{equation}
\hat\theta_{COND} = \arg\min_{\theta\in\overline\Theta}{K_{COND}(\tilde\theta,\theta)} \; .\label{eq:condpoint}
\end{equation}
By equations (\ref{eq:map}) and (\ref{eq:cond}) we may write
\begin{equation}\label{eq:condwins}
K_{JOINT}(\tilde\theta,\theta)=K_{COND}(\tilde\theta,\theta)+
\int{p(\x\mid \tilde\theta)\log\frac{p(\x\mid \tilde\theta)}
{p(\x\mid \theta)} \, d\x} ~~~.
\end{equation}
Therefore the point estimates $\hat\theta_{JOINT}$ and
$\hat\theta_{COND}$ are different in general.  If the model that has
generated the data does not belong to the model family, that is
$\tilde\theta\notin\overline\Theta$, then by Equation
(\ref{eq:condwins}) the joint estimate is generally worse than the
conditional estimate in conditional inference, measured in terms of
conditional likelihood. See also \cite{Nadas88}.

\subsection{Discriminative versus generative models}

A discriminative model does not make any assumptions on the
distribution of the margin of $\x$. That is, it does not incorporate a
generative model for $\x$, and can be interpreted to rather use the
empirical distribution of $p(\x)$ as its margin \cite{Hastie01}. 

A generative model, on the other hand, assumes a specific parametric
form for the full distribution $p(c,\x,\theta)$. A generative model
can be learned either as a joint density model or in a discriminative
manner.  Our point in this paper is that the selection corresponds to
choosing the utility function; this is actually our fifth axiom in
section \ref{sec:axiomatic}. In joint density modeling the utility is
to model the full distribution $p(c,\x)$ as accurately as possible,
which corresponds to computing the standard posterior incorporating
the likelihood function. In discriminative learning the utility is to
model the conditional distribution of $p(c\mid \x)$, and the result is
a discriminative posterior incorporating the conditional likelihood.
A generative model optimized in discriminative manner is referred to
as a discriminative joint density model in the following. 

For generative models and in case of point estimates, if the model
family is correct, a maximum likelihood (ML) solution is better for
predicting $p(c|\x)$ than maximum conditional likelihood (CML). They
have the same maximum, but the asymptotic variance of CML estimate is
higher \cite{Nadas83}.  However, in case of incorrect models, a
maximum conditional likelihood estimate is better than maximum
likelihood \cite{Nadas88}. For an example illustrating that CML can be
better than ML in predicting $p(c|\x)$, see \cite{Jebara00}.

Since the discriminative joint density model has a more restricted
model structure than Bayesian regression, we expect it to perform
better with small amounts of data. More formally, later in Theorem
\ref{th:posterior} we move from point estimate of Equation
(\ref{eq:condpoint}) to a discriminative posterior distribution
$p_d(\theta\mid D)$ over the model parameters $\theta$ in
$\overline\Theta$. Since the posterior is normalized to unity,
$\int_{\overline\Theta} {p_d(\theta\mid D)} d\theta=1$, the values of
the posterior $p_d(\theta\mid D)$ are generally smaller for larger
model families; the posterior is more diffuse. Equation
(\ref{eq:cond}) can be generalized to the expectation of the
approximation error,
\begin{equation}
E_{p_d(\theta\mid D)}\left[K_{COND}(\tilde\theta,\theta)\right]
=-\int{\sum_{c}{p(\x,c\mid\tilde\theta)p_d(\theta\mid D)\log{p(c\mid\x,\theta)}}d\x d\theta}+{\rm const.} 
\label{eq:expK}
\end{equation}
The expected approximation error is small when both the discriminative
posterior distribution $p_d(\theta\mid D)$ and the conditional
likelihood $p(c\mid\x,\theta)$ are large at the same time. For small
amounts of data, if the model family is too large, the values of the
posterior $p_d(\theta\mid D)$ are small. The discriminative joint
density model has a more restricted model family than that of the
Bayesian regression, and hence the values of the posterior are larger.
If the model is approximately correct, the discriminative joint
density model will have a smaller approximation error than the
Bayesian regression, that is, $p(c\mid\x,\theta)$ is large somewhere
in $\overline\Theta$. This is analogous to selecting the model family
that maximizes the evidence (in our case the expected conditional
log-likelihood) in Bayesian inference; choosing a model family that is
too complex leads to small evidence (see, e.g., \cite{Bishop06}). The
difference to the traditional Bayesian inference is that we do not
require that the true data generating distribution $\tilde\theta$ is
contained in the parameter space $\overline\Theta$ under
consideration.


\section{Axiomatic Derivation of Discriminative Posterior}
\label{sec:axiomatic}

In this section we generalize the point estimate $\hat\theta_{COND}$
presented in section \ref{sec:point estimates} to a {\em discriminative
  posterior distribution} over $\theta\in\overline\Theta$.
\begin{theorem}[Discriminative posterior distribution]
\label{th:posterior}
It follows from axioms 1--6 listed below that, given data $D = \{
(c_i,\x_i) \}_{i=1}^n$, the discriminative posterior distribution
$p_{d}(\theta\mid D)$ is of the form
\begin{equation}
p_{d}(\theta\mid D)\propto p(\theta)\prod_{(c,\x)\in D}{p(c\mid 
\x,\theta)}~~.
\nonumber
\label{eq:posterior}
\end{equation}
The predictive distribution for new $\tilde \x$, obtained by
integrating over this posterior, $p(c\mid \tilde
\x,D)=\int{p_d(\theta\mid D) p(c\mid \tilde \x,\theta)\, d\theta}$, is
consistent for conditional inference. That is, $p_{d}$ is consistent
for the utility of conditional likelihood.
\end{theorem}

The discriminative posterior follows from requiring the following
axioms to hold:
\begin{enumerate}
\item The posterior $p_d(\theta\mid D)$ can be represented by
  non-negative real numbers that satisfy
  $\int_{\overline\Theta} {p_d(\theta\mid D)} d\theta=1$.

\item A model $\theta\in\overline\Theta$ can be represented as a
  function $h((c,\x),\theta)$ that maps the observations $(c,\x)$ to
  real numbers.

\item The posterior, after observing a data set $D$ followed by an
  observation $(c,\x)$, is given by $p_d(\theta\mid D\cup (c,{\bf
    x}))=F(h((c,\x),\theta),p_d(\theta\mid D))$, where $F$ is a twice
  differentiable function in both of its parameters.

\item Exchangeability: The value of the posterior is independent of
  the ordering of the observations. That is, the posterior after two
  observations $(c,\x)$ and $(c',\x')$ is the same
  irrespective of their ordering: \\
  $F(h((c',\x'),\theta),p_d(\theta\mid
  (c,\x)\cup D)) = F(h((c,\x),\theta),p_d(\theta\mid
  (c',\x')\cup D))$.

\item The posterior must agree with the utility. For
  $\tilde\theta\in\Theta$, and $\theta_1,\theta_2\in\overline\Theta$,
  the following condition is satisfied:
\begin{displaymath}
p_{d}(\theta_1\mid D_{\tilde\theta}) \le p_{d}(\theta_2\mid D_{\tilde\theta}) \Leftrightarrow
K_{COND}(\tilde\theta,\theta_1)\ge K_{COND}(\tilde\theta,\theta_2)~~~,
\end{displaymath}
where $D_{\tilde\theta}$ is a very large data set sampled from
$p(c,\x\mid \tilde\theta)$. We further assume that the
discriminative posteriors $p_d$ at $\theta_1$ and $\theta_2$ are equal
only if the corresponding conditional KL-divergences $K_{COND}$ are
equal.

\end{enumerate}
The first axiom above is simply a requirement that the posterior is a
probability distribution in the parameter space. 

The second axiom defines a model in general terms; we define it as a
mapping from event space into a real number.

The third axiom makes smoothness assumptions on the posterior. The
reason for the axiom is technical; the smoothness is used in the
proofs. Cox \cite{Cox46} makes similar assumptions, and our proof
therefore holds in the same scope as Cox's; see \cite{Halpern99}.

The fourth axiom requires exchangeability; the shape of the posterior
distribution should not depend on the order in which the observations
are made. This deviates slightly from analogous earlier proofs for
standard Bayesian inference, which have rather used the requirement of
associativity \cite{Cox46} or included it as an additional constraint
in modeling after presenting the axioms \cite{Bernardo00}.

The fifth axiom states, in essence, that asymptotically (at the limit
of a large but finite data set) the shape of the posterior is such
that the posterior is always smaller if the ``distance''
$K_{COND}(\tilde\theta,\theta)$ is larger. If the opposite would be
true, the integral over the discriminative posterior would give larger
weight to solutions further away from the true model, leading to a
larger error measured by $K_{COND}(\tilde\theta,\theta)$.

Axioms 1--5 are sufficient to fix the discriminative posterior, up to
a monotonic transformation. To fix the monotonic transformation we
introduce the sixth axiom:
\begin{enumerate}
\item[6.] For fixed $\x$ the model reduces to the standard posterior. For the
data set $D_x=\{(c,\x')\in D\mid \x'=
\x\}$, the discriminative
posterior $p_d(\theta\mid D_x)$ matches the standard posterior\\
 $p^x(c\mid \theta)\equiv p(c\mid\x,\theta)$.
\end{enumerate}
We use $p(\theta)=p_d(\theta\mid \emptyset)$ to denote the posterior
when no data is observed ({\em prior distribution}).

\begin{proof}
The proof is in the Appendix.  For clarity
of presentation, we additionally sketch the proof in the following.
\begin{proposition}[$F$ is isomorphic to multiplication]
  It follows from axiom 4
  that the function $F$ is of the form
  \begin{equation}
    f(F(h((c,\x),\theta),p_d(\theta\mid D))\propto h((c,\x),\theta) f(p_d(\theta\mid D)),
  \end{equation}
\end{proposition}
where $f$ is a monotonic function which we, by convention, fix to the
identity function.\footnote{We follow here Cox \cite{Cox46} and
  subsequent work, see e.g. \cite{Halpern99}. A difference which
  does not affect the need for the convention is that usually
  multiplicativity is derived based on the assumption of
  associativity, not exchangeability.  Our setup is slightly different
  from Cox, since instead of updating beliefs on events, we here
  consider updating beliefs in a family of models.}

The problem then reduces to finding a functional form for
$h((c,\x),\theta)$.  Utilizing both the equality part and the
inequality part of axiom 5, the following proposition can be derived.
\begin{proposition}
  It follows from axiom 5 that
\begin{equation}\label{eq:hlin}
  h((c,\x),\theta)
\propto p(c\mid \x,\theta)^A \quad \rm{where} 
\quad A>0~~~.
\end{equation}
\end{proposition}
Finally, the axiom 6 effectively states that we decide to follow the
Bayesian convention for a fixed $\x$, that is, to set $A=1$.
\end{proof}

\section{Modeling Missing Data}

\begin{claim}[New]
  Discriminative posterior gives a theoretically justified way of
  handling missing values in discriminative tasks.
\end{claim}
Discriminative models cannot readily handle components missing from
the data vector, since the data is used only as covariates. However,
standard methods of handling missing data with generative models
\cite{Little02} can be applied with the discriminative posterior.

The additional assumption we need to make is a model for
which data are missing. Below we derive the formulas for the
common case of data missing independently at random. Extensions
incorporating prior information of the process by which the data are
missing are straightforward, although possibly not trivial.

\begin{sloppypar}
Write the observations $\x=(\x_1,\x_2)$. Assume that $\x_1$ can be
missing and denote a missing observation by $\x_1=\oslash$. The task is still
to predict $c$ by $p(c\mid \x_1,\x_2)$. 
\end{sloppypar}

Since we are given a model for the joint distribution which is assumed
to be approximately correct, it will be used to model the missing
data.  We denote this by $q(c,\x_1,\x_2\mid \theta')$,
$\x_1\ne\oslash$, with a prior $q(\theta')$,
$\theta'\in\overline\Theta'$. We further denote the parameters of the
missing data mechanism by $\lambda$. Now, similar to joint density
modeling, if the priors for the joint model, $q(\theta')$, and
missing-data mechanism, $g(\lambda)$, are independent, the missing
data mechanism is ignorable \cite{Little02}. In other words, the
posterior that takes the missing data into account can be written as
\begin{equation}
p_{d}(\theta\mid D) \propto
 q(\theta') g(\lambda)
\prod_{{\bf y} \in D_{full}}
{q(c\mid \x_1,\x_2,\theta')} 
\prod_{{\bf y}\in D_{missing}}{q(c\mid \x_2,\theta')} ,
\label{eq:missing}
\end{equation}
where ${\bf y}= (c,\x_1,\x_2)$ and we have used to $D_{full}$ and $D_{missing}$ 
to denote the portions of data set with $\x_1\ne\oslash$ and
$\x_1=\oslash$, respectively.

Equation~(\ref{eq:missing}) has been obtained by using $q$ to
construct a model family in which the data is missing independently at
random with probability $\lambda$, having a prior $g(\lambda)$. That
is, we define a model family that generates the missing data in
addition to the non-missing data,
\begin{equation}
 p(c,\x_1,\x_2\mid \theta)\equiv \left\{ \begin{array}{lcl}
 (1-\lambda)q(c,\x_1,\x_2\mid \theta') & , &\x_1\ne\oslash \\
 \lambda \int_{\mathrm{supp}(\x_1)\setminus \oslash}{q(c,\x,\x_2\mid \theta')}d\x=\lambda q(c,\x_2\mid \theta')
 & , & \x_1=\oslash ,
 \end{array}\right.
\label{eq:missingp}
\end{equation}
where $\theta\in\overline\Theta$ and $\mathrm{supp}(\x_1)$ denotes the
support of $\x_1$.  The parameter space $\overline\Theta$ is spanned by
$\theta'\in\overline\Theta'$ and $\lambda$.  The
equation~(\ref{eq:missing}) follows directly by applying
Theorem~\ref{th:posterior} to $p(c\mid
\x_1,\x_2,\theta)$.\footnote{Notice that $p(c\mid
  \x_1,\x_2,\theta)=q(c\mid
  \x_1,\x_2,\theta')$ when $\x_1\ne\oslash$, and\\
  $p(c\mid \x_1=\oslash,\x_2,\theta)=q(c\mid \x_2,\theta')$.} Notice
that the posterior for $\theta'$ of~(\ref{eq:missing}) is independent
of $\lambda$. The division to $\x_1$ and $\x_2$ can be made separately
for each data item. The items can have different numbers of missing
components, each component $k$ having a probability $\lambda_k$ of
being missing.

\section{Experiments}
 
\subsection{Implementation of Discriminative Sampling}
\label{sec:sampling}

The discriminative posterior can be sampled with an ordinary
Metropolis-Hastings algorithm where the standard posterior
$p(\theta\mid D) \propto \prod_{i=1}^n p(c_i,x_i\mid \theta) p(\theta)$
is simply replaced by the discriminative version $p_{d}(\theta\mid D)
\propto \prod_{i=1}^n p(c_i \mid x_i, \theta) p(\theta)$, where $p(c_i
\mid x_i, \theta) = \frac{p(c_i,x_i\mid \theta)}{p(x_i\mid \theta)}$.

In MCMC sampling of the discriminative posterior, the normalization
term of the conditional likelihood poses problems, since it involves a
marginalization over the class variable $c$ and latent variables $z$,
that is, $p(\x_i \mid \theta)=\sum_c\int_{\mathrm{supp}(z)} p(\x_i,z,c
\mid \theta) dz$. In case of discrete latent variables, such as in
mixture models, the marginalization reduces to simple summations and
can be computed exactly and efficiently. If the model contains
continuous latent variables the integral needs to be evaluated
numerically.

\subsection{Performance as a function of the incorrectness of the model}

We first set up a toy example where the distance between the model
family and the true model is varied.

\subsubsection{Background}
\label{sec:logreg vs mm}
In this experiment we compare the performance of logistic regression
and a mixture model. Logistic regression was chosen since many of the
discriminative models can be seen as extensions of it, for
example conditional random fields \cite{Sutton06}. In case of logistic
regression, the following theorem exists:
\begin{theorem}\cite{Banerjee07logreg}\\
  For a k-class classification problem with equal priors, the pairwise
  log-odds ratio of the class posteriors is affine if and only if the
  class conditional distributions belong to any fixed exponential
  family.
\end{theorem} 
That is, the model family of logistic regression incorporates the
model family of any generative exponential (mixture)
model\footnote{Banerjee \cite{Banerjee07logreg} provides a proof also
  for conditional random fields.}. A direct interpretation is that any
generative exponential family mixture model defines a smaller subspace
within the parametric space of the logistic regression model.

Since the model family of logistic regression is larger, it is
asymptotically better, but if the number of data is small, generative
models can be better due to the restricted parameter space; see for
example \cite{Ng02} for a comparison of naive Bayes and logistic
regression. This happens if the particular model family is
at least approximately correct.

\subsubsection{Experiment}
The true model generates ten-dimensional data $\x$ from the
two-component mixture $p(\x)=\sum_{j=1}^2 \frac 12\times p(\x\mid
\tilde \bmu_j,\tilde \bsigma_j), $ where $j$ indexes the mixture
component, and $p(\x\mid \tilde \bmu_j,\tilde \bsigma_j)$ is a
Gaussian with mean $\tilde \bmu_j$ and standard deviation $\tilde
\bsigma_j$.  The data is labeled according to the generating mixture
component (i.e., the ``class'' variable) $j \in \left\{1,2
\right\}$. Two of the ten dimensions contain information about the
class. The ``true'' parameters, used to generate the data, on these
dimensions are $\tilde \mu_1=5,\ \tilde \sigma_1=2$, and $\tilde
\mu_2=9,\ \tilde \sigma_2=2.$ The remaining eight dimensions are
Gaussian noise with $\tilde \mu=9,\ \tilde \sigma=2$ for both
components.

The ``incorrect'' generative model used for inference is a mixture of
two Gaussians where the variances are constrained to be
$\sigma_j=k\cdot \mu_j + 2$. With increasing $k$ the model family thus
draws further away from the true model. We assume for both of the
$\mu_j$ a Gaussian prior $p(\mu_j\mid m,s)$ having the same fixed
hyperparameters\footnote{The model is thus slightly incorrect even
  with $k=0$.} $m=7, s=7$ for all dimensions.

Data sets were generated from the true model with 10000 test data
points and a varying number of training data points
($N_D=\{32,64,128,256,512,1024\}$).  Both the discriminative and
standard posteriors were sampled for the model. The goodness of the
models was evaluated by perplexity of conditional likelihood on the
test data set.

Standard and discriminative posteriors of the incorrect generative
model are compared against Bayesian logistic regression. A uniform
prior for the parameters $\beta$ of the logistic regression model was
assumed.  As mentioned in subsection \ref{sec:logreg vs mm} above, the
Gaussian naive Bayes and logistic regression are connected (see for
example \cite{Mitchell06} for exact mapping between the
parameters). The parameter space of the logistic model incorporates
the true model as well as the incorrect model family, and is therefore
the optimal discriminative model for this toy data. However, since the
parameter space of the logistic regression model is larger, the model
is expected to need more data samples for good predictions.

The models perform as expected (Figure \ref{fig:toycomparison}).  The
model family of our incorrect model was chosen such that it contains
useful prior knowledge about the distribution of $\x$. Compared with
logistic regression, the model family becomes more restricted which is
beneficial for small learning data sets (see
Figure~\ref{fig:toycomparison}a).  Compared to joint density sampling,
discriminative sampling results in better predictions of the class $c$
when the model family is incorrect.

The models were compared more quantitatively by repeating the sampling
ten times for a fixed value of $k=2$ for each of the learning data set
sizes; in every repeat also a new data set was generated. The results
in Figure \ref{fig:toycomparison2} confirm the qualitative findings of
Figure \ref{fig:toycomparison}.

The posterior was sampled with Metropolis-Hastings algorithm with a
Gaussian jump kernel. In case of joint density and discriminative
sampling, three chains were sampled with a burn-in of 500 iterations
each, after which every fifth sample was collected. Bayesian
regression required more samples for convergence, so a burn-in of 5500
samples was used. Convergence was estimated when carrying out
experiments for Figure \ref{fig:toycomparison2}; the length of
sampling chains was set such that the confidence intervals for each of
the models was roughly the same.  The total number of samples was 900
per chain. The width of jump kernel was chosen as a linear function of
data such that the acceptance rate was between 0.2--0.4 \cite{Gelman03}
as the amount of data was increased. Selection was carried out by
preliminary tests with different random seeds.

\paragraph{Missing Data.}
The experiment with toy data is continued in a setting where 50\% of
the learning data are missing at random.

MCMC sampling was carried out as described above.  For the logistic
regression model, a multiple imputation scheme is applied, as
recommended in~\cite{Gelman03}. In order to have sampling conditions
comparable to sampling from a generative model, each sampling chain
used one imputed data set. Imputation was carried out with the
generative model (representing our current best knowledge, differing
from the true joint model, however).

As can be seen from Figure~\ref{fig:toycomparison}, discriminative
sampling is better than joint density sampling when the model is
incorrect. The performance of Bayesian regression seems to be affected
heavily by the incorrect generative model used to generate missing
data. As can be seen from Figure~\ref{fig:toycomparison2},
surprisingly the Bayesian regression is even worse than standard
posterior. The performance could be increased by imputing more than
one data sets with the cost of additional computational complexity,
however.

\begin{figure}[h]
\begin{center}
\begin{tabular}{cccc}
& Joint density MCMC & Discriminative MCMC & Bayesian regression\\ 
\raisebox{2cm}{{\large a)}} & 
\includegraphics[width=0.25\textwidth]{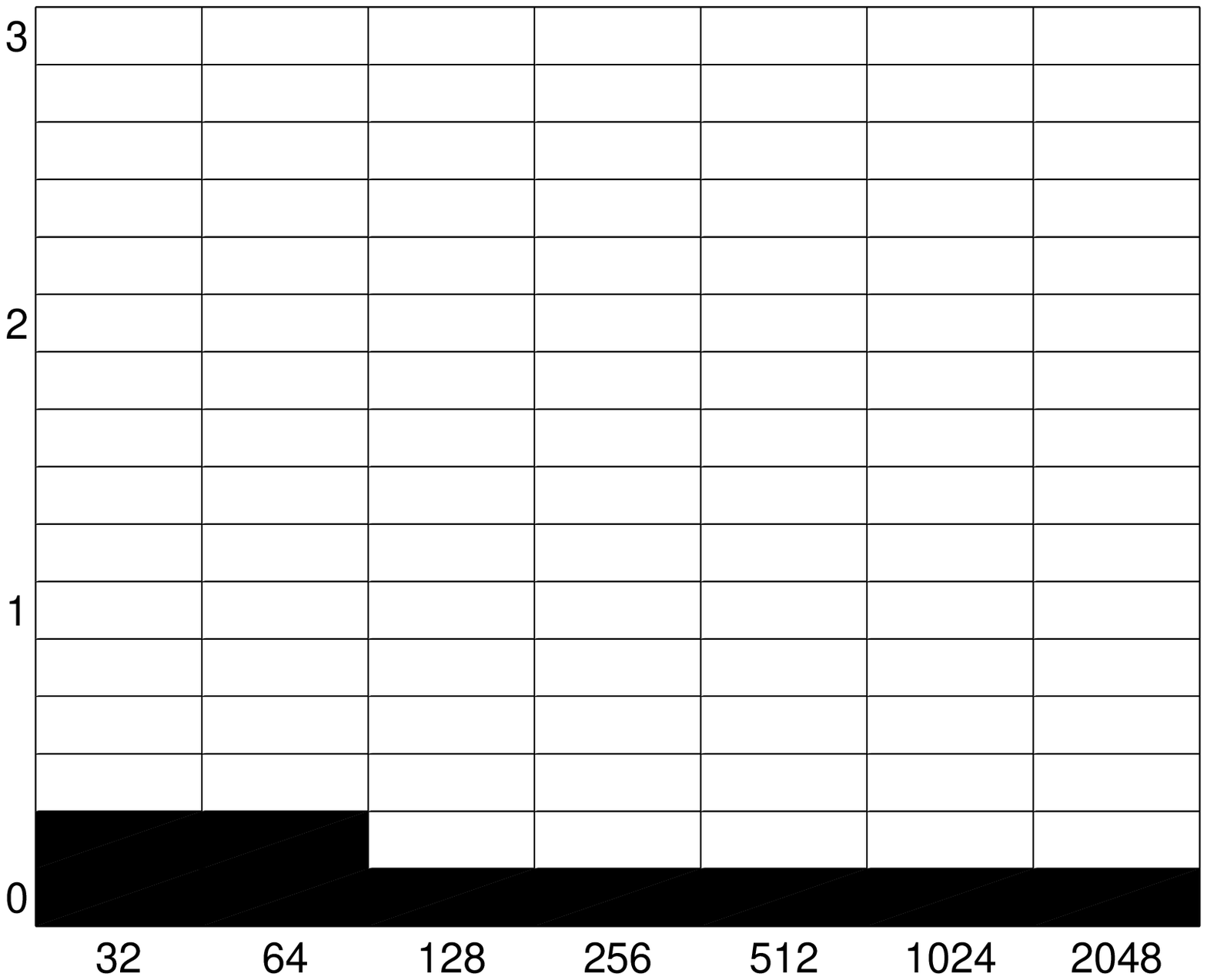}&
\includegraphics[width=0.25\textwidth]{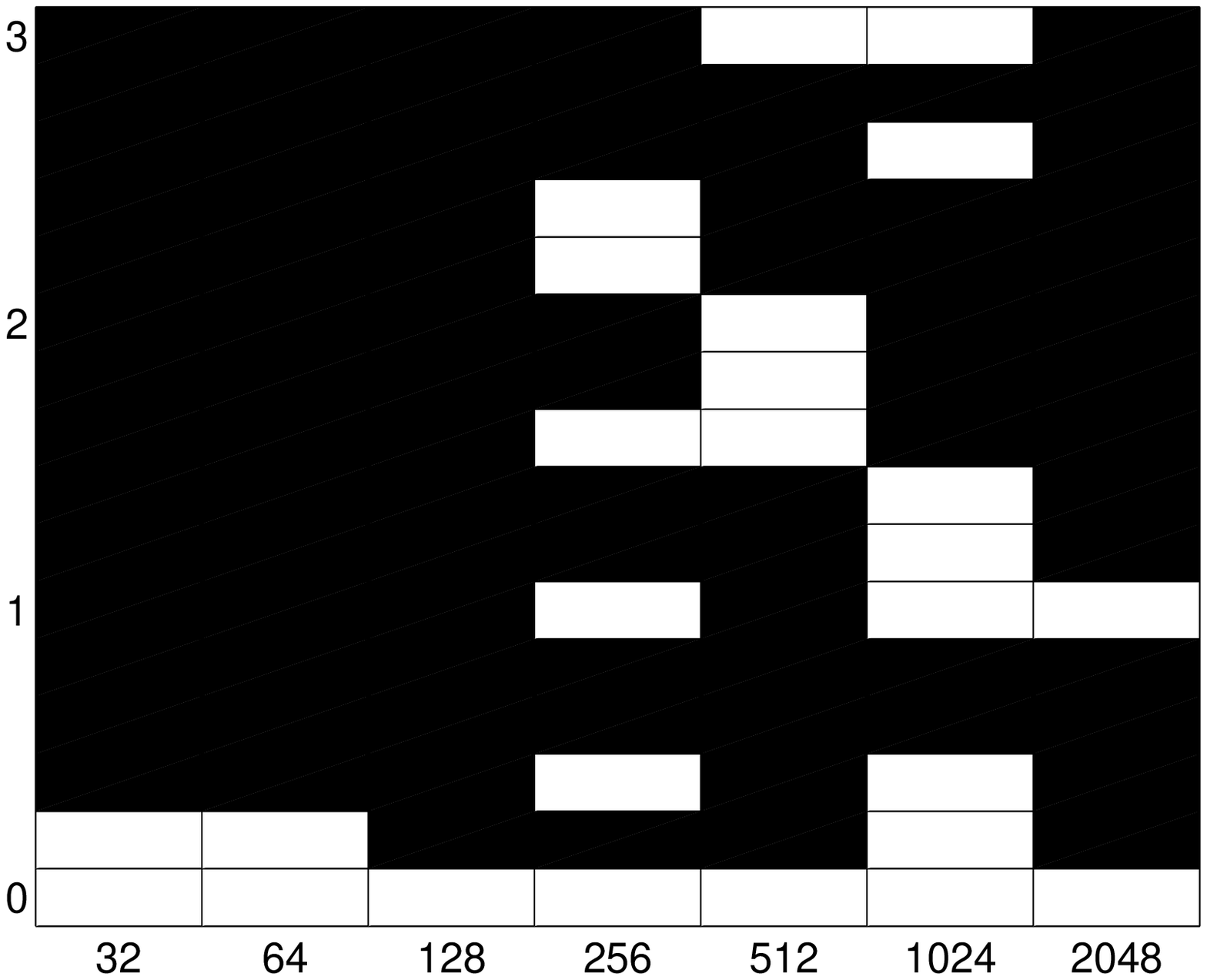}&
\includegraphics[width=0.25\textwidth]{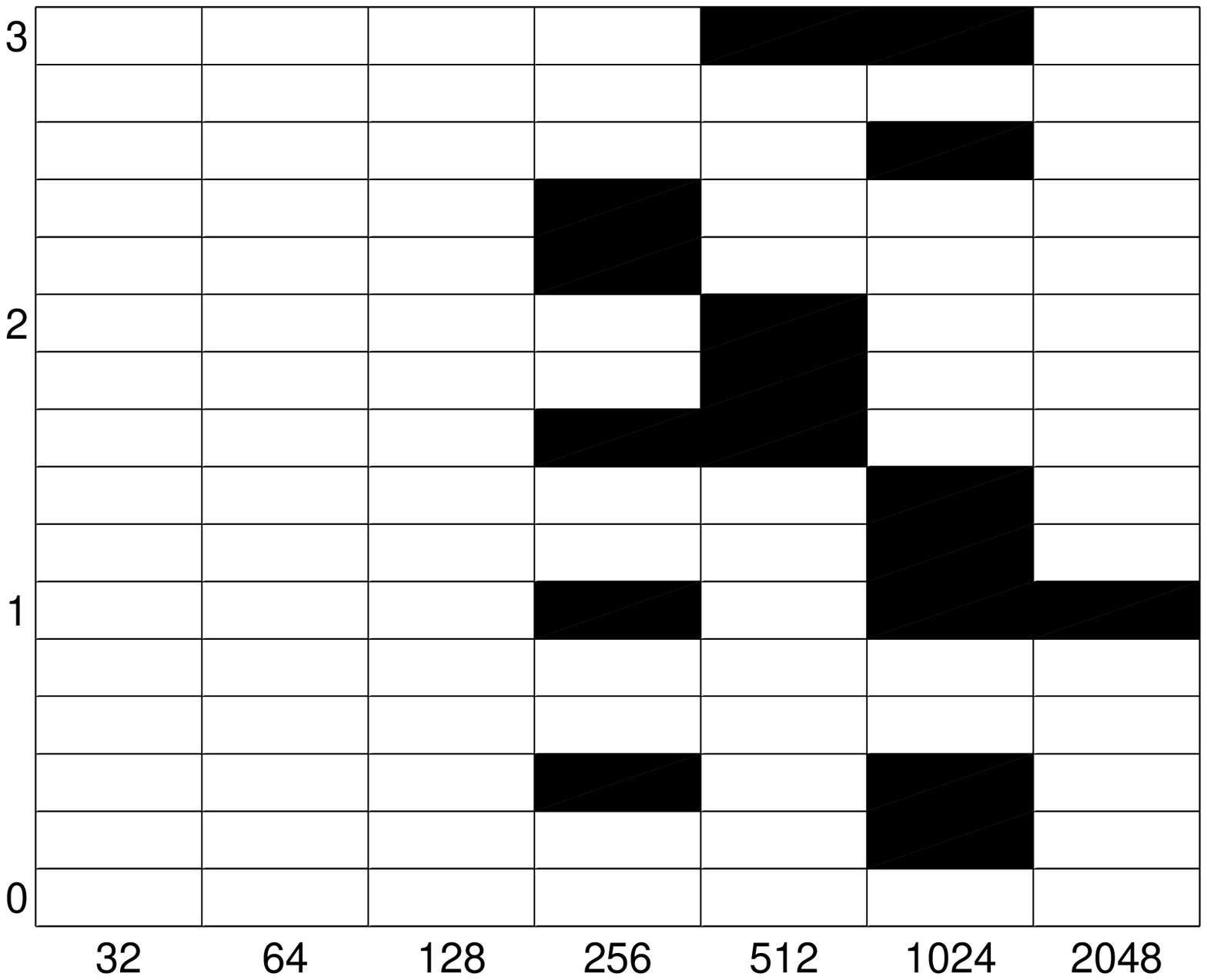}\\
\raisebox{2cm}{{\large b)}} & 
\includegraphics[width=0.25\textwidth]{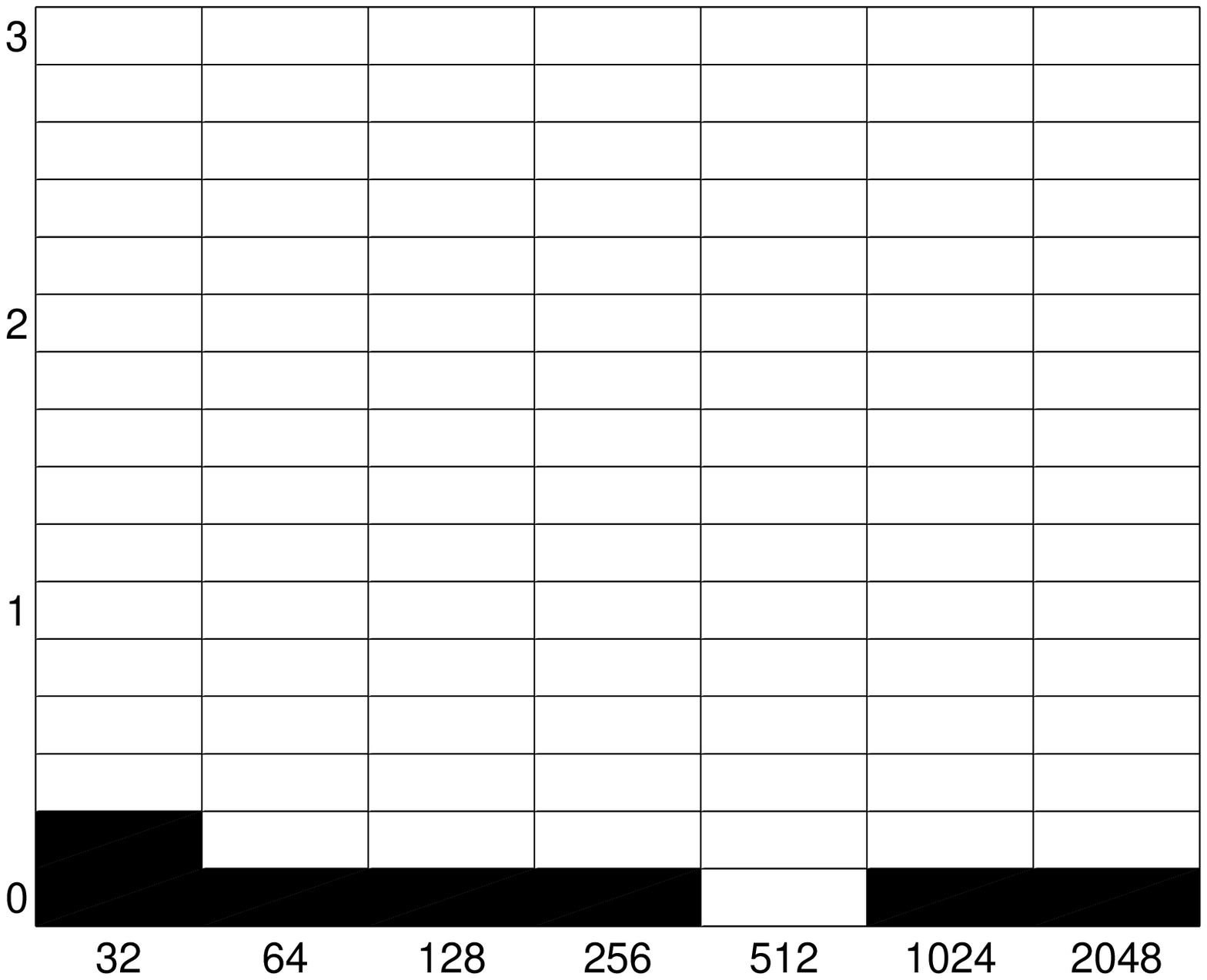}&
\includegraphics[width=0.25\textwidth]{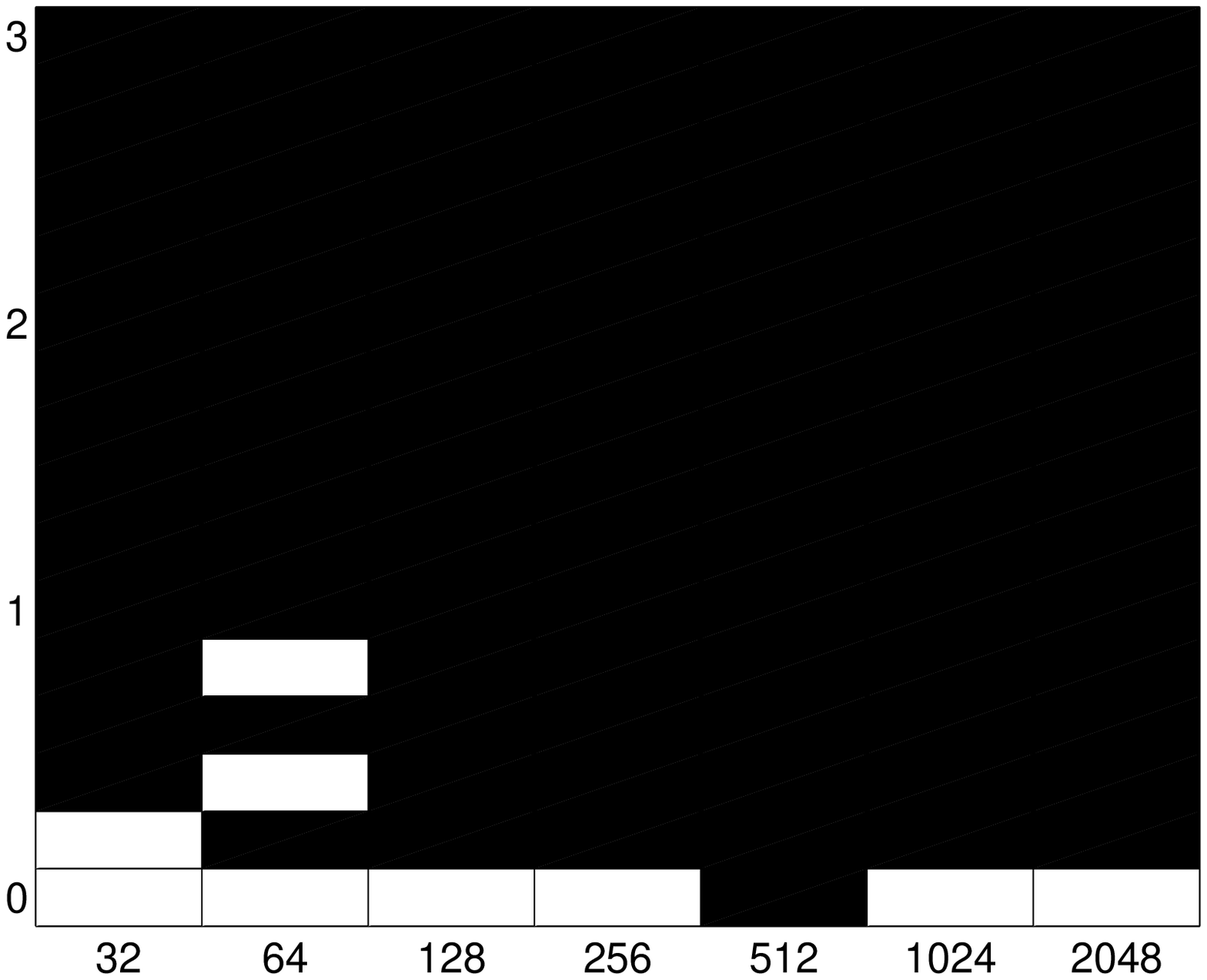}&
\includegraphics[width=0.25\textwidth]{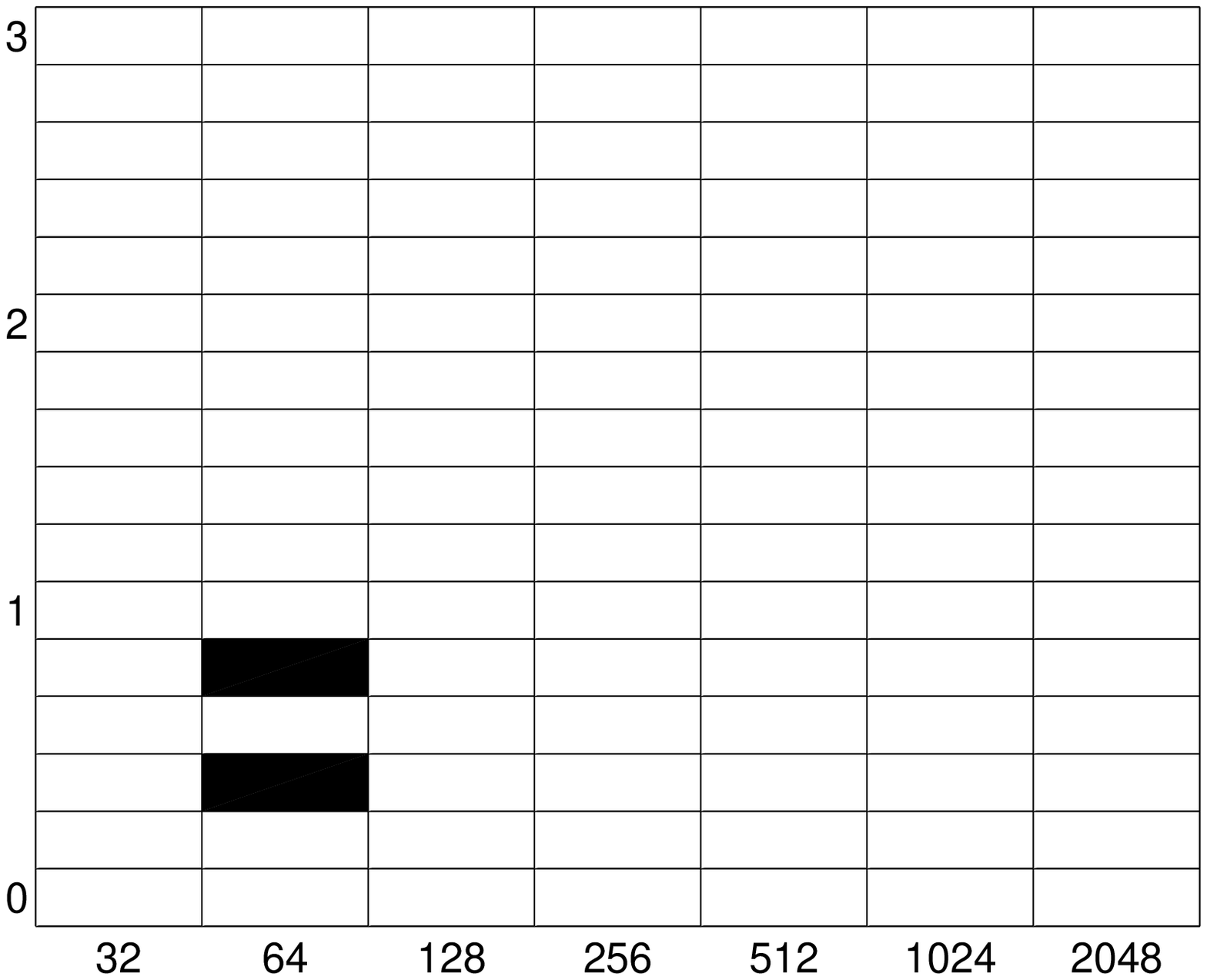}
\end{tabular}
\end{center}
\caption{A comparison of joint density sampling, discriminative
  sampling, and Bayesian regression, a) Full data, b) 50\% of the data
  missing. Grid points where the method is better than others are
  marked with black. The X-axis denotes the amount of data, and Y-axis
  the deviance of the the model family from the ``true'' model (i.e.,
  the value of $k$). The methods are prone to sampling error, but the
  following general conclusions can be made: Bayesian generative
  modeling (``Joint density MCMC'') is best when the model family is
  approximately correct. Discriminative posterior (``Discriminative
  MCMC'') is better when the model is incorrect and the learning data
  set is small. As the amount of data is increased, Bayesian
  regression and discriminative posterior show roughly equal
  performance (see also Figure \ref{fig:toycomparison2}).}
\label{fig:toycomparison}
\end{figure}

\begin{figure}[h]
\begin{center}
\begin{tabular}{cc}
{\large a)} & {\large b)} \\ 
\includegraphics[width=0.45\textwidth]{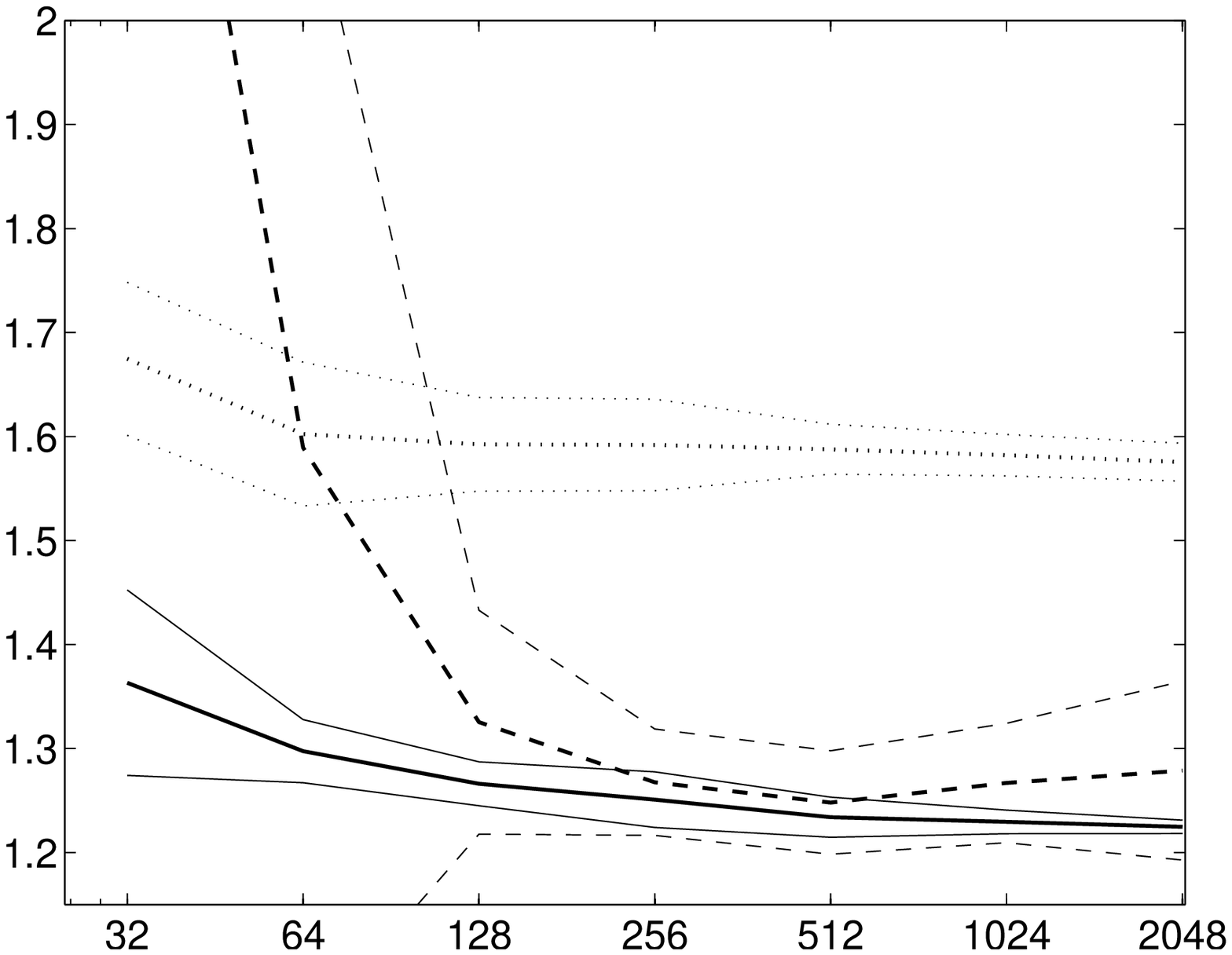}&
\includegraphics[width=0.45\textwidth]{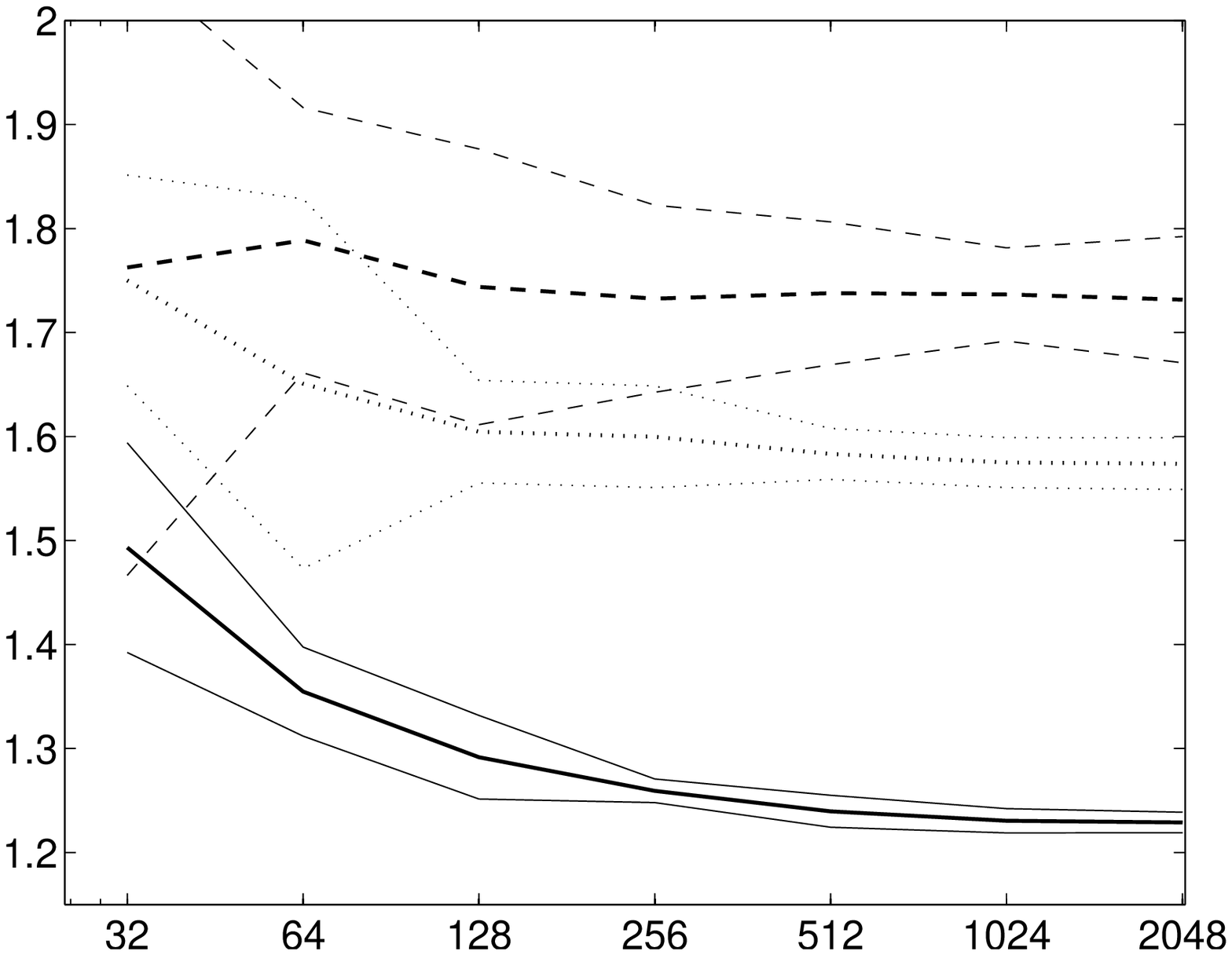}\\
\end{tabular}
\end{center}
\caption{A comparison of joint density sampling (dotted line),
  discriminative sampling (solid line), and Bayesian regression
  (dashed line) with an incorrect model. X-axis: learning data set
  size, Y-axis: perplexity. Also the 95 \% confidence intervals are
  plotted (with the thin lines). Logistic regression performs
  significantly worse than discriminative posterior with small data
  set sizes, whereas with large data sets the performance is roughly
  equal. Joint density modeling is consistently worse. The model was
  fixed to $k=2$, ten individual runs were carried out in order to
  compute 95\% confidence intervals. a) Full data, b) 50\% of the data
  missing.}
\label{fig:toycomparison2}
\end{figure}
 
\subsection{Document Modeling}

As a demonstration in a practical domain we applied the discriminative
posterior to document classification. We used the Reuters data
set~\cite{Lewis04}, of which we selected a subset of 1100 documents
from four categories (CCAT, ECAT, GCAT, MCAT).  Each selected document
was classified to exactly one of the four classes. As a preprocessing
stage, we chose the 25 most informative words within the training set
of 100 documents, having the highest mutual information between
classes and words~\cite{peltonen04icml}. The remaining 1000 documents
were used as a test set.

We first applied a mixture of unigrams model (MUM, see
Figure~\ref{fig:ddpca} left)~\cite{Nigam00}.  The model assumes that
each document is generated from a mixture of $M$ hidden ``topics,'' $
p(\x_i \mid \theta)=\sum_{j=1}^M \pi(j) p(\x_i|\bbeta_j)$, where $j$
is the index of the topic, and $\bbeta_j$ the multinomial parameters
that generate words from the topic. The vector $\x_i$ contains the
observed word counts (with a total of $N_W$) for document $i$, and
$\pi(j)$ is the probability of generating words from the topic
$j$. The usual approach \cite{Gelman03} for modeling paired data
$\left\{\x_i,c_i\right\}_{i=1}^{N_D}$ by a joint density mixture model
was applied; $c$ was associated with the label of the mixture
component from which the data is assumed to be generated. Dirichlet
priors were used for the ${\mathbf \beta}, \pi$, with hyperparameters
set to 25 and 1, respectively. We used the simplest form of MUM
containing one topic vector per class. The sampler used
Metropolis-Hastings with a Gaussian jump kernel. The kernel width was
chosen such that the acceptance rate was roughly 0.2 \cite{Gelman03}.

As an example of a model with continous hidden variables, we
implemented the Latent Dirichlet Allocation (LDA) or discrete PCA
model \cite{Blei02,Buntine06book}.  We constructed a variant of the
model that generates also the classes, shown in Figure~\ref{fig:ddpca}
(right). The topology of the model is a mixture of LDAs; the
generating mixture component $z_c$ is first sampled from $\pi_c$. The
component indexes a row in the matrix $\alpha$, and (for simplicity)
contains a direct mapping to $c$. Now, given $\alpha$, the generative
model for words is an ordinary LDA. The $\pi$ is a topic distribution
drawn individually for each document $d$ from a Dirichlet with
parameters $\alpha\left(z_c\right)$, that is,
Dirichlet$\left(\alpha\left(z_c\right)\right)$. Each word $w_{nd}$
belongs to one topic $z_{nd}$ which is picked from
Multinomial$\left(\pi\right)$. The word is then generated from
Multinomial$\left(\beta\left(z_{nd},\cdot\right)\right)$, where
$z_{nd}$ indexes a row in the matrix ${\mathbf \beta}$. We assume
a Dirichlet prior for the ${\mathbf \beta}$, with hyperparameters set to
2, Dirichlet prior for the ${\mathbf \alpha}$, with hyperparameters
set to 1, and a Dirichlet prior for $\pi_c$ with hyperparameters
equal to 50. The parameter values were set in initial test runs (with
a separate data set from the Reuters corpus). Four topic vectors were
assumed, making the model structure similar to \cite{FeiFei05}. The
difference to MUM is that in LDA-type models a document can be
generated from several topics.

Sampling was carried out using Metropolis-Hastings with a Gaussian jump
kernel, where the kernel width was chosen such that the acceptance
rate was roughly 0.2 \cite{Gelman03}.  The necessary integrals were
computed with Monte Carlo integration. The convergence of integration
was monitored with a jackknife estimate of standard
error \cite{Efron93}; sampling was ended when the estimate was less
than 5 \% of the value of the integral. The length of burn-in was 100
iterations, after which every tenth sample was picked. The total
number of collected samples was 100. The probabilities were clipped to
the range $[e^{-22},1]$.

\begin{figure}[h]
\begin{center}
\includegraphics*[width=0.22\textwidth,angle=270]{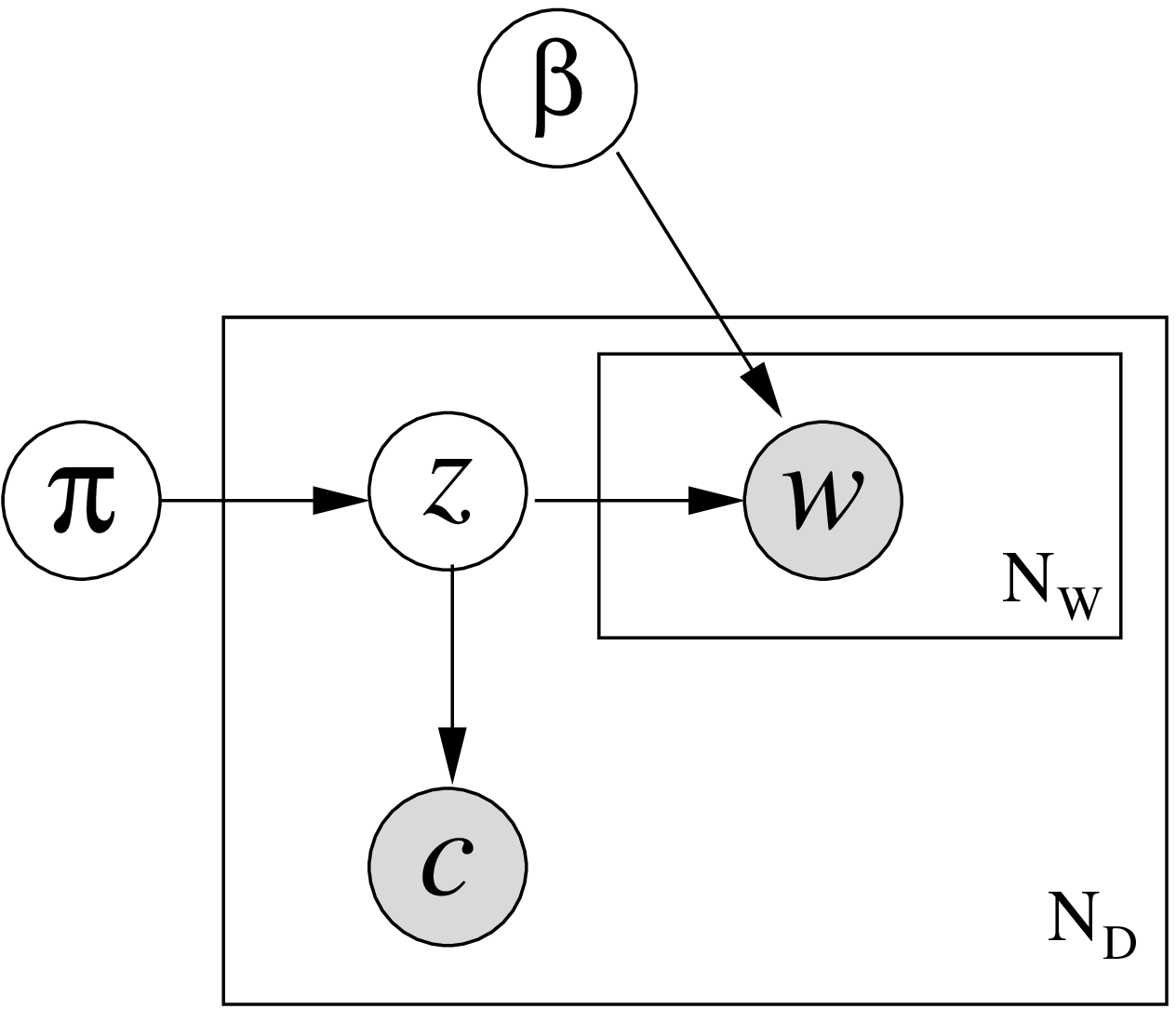} \hspace{1 cm}
\includegraphics*[width=0.22\textwidth,angle=270]{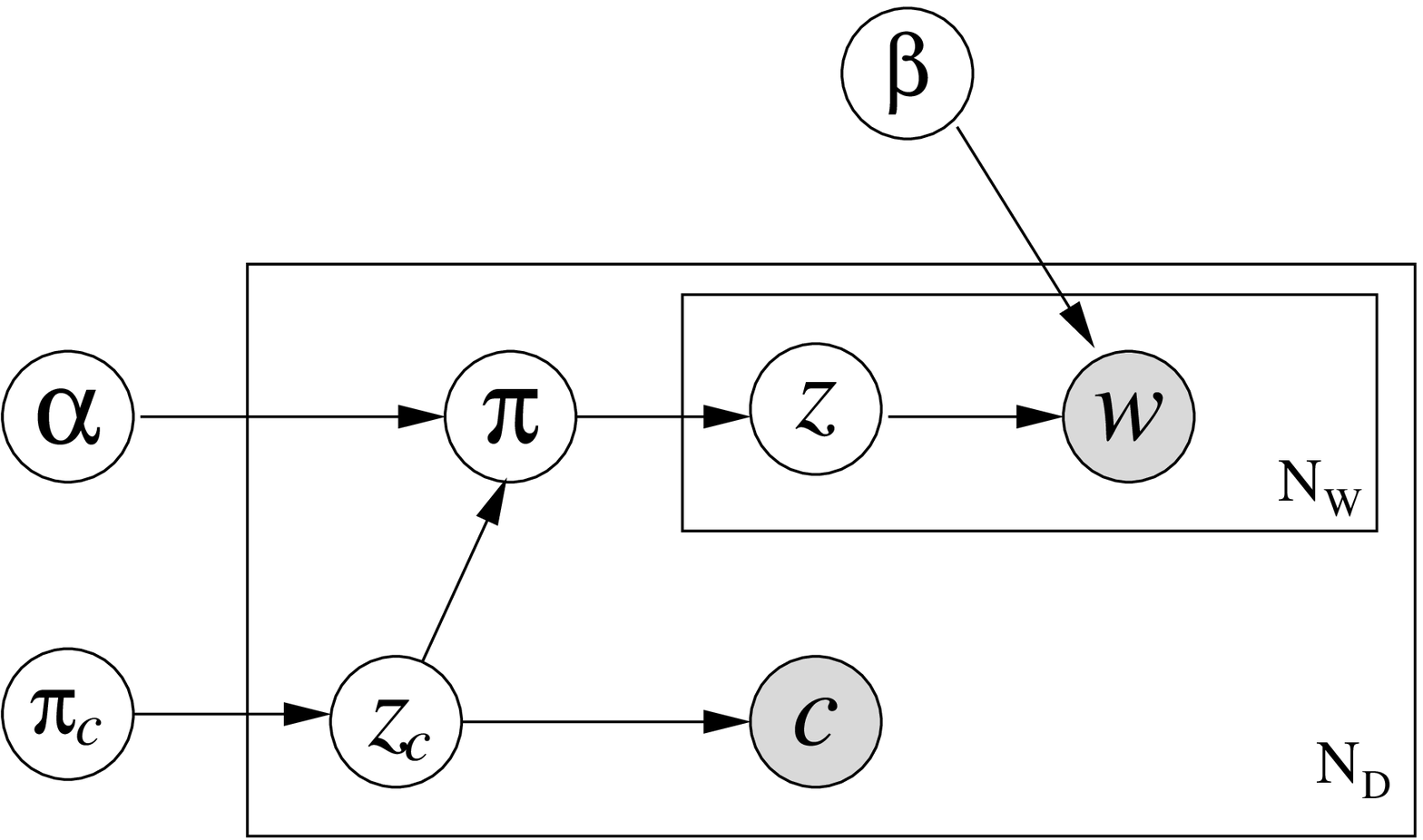} 
\end{center}
\caption{Left: Mixture of Unigrams. Right: Mixture of Latent Dirichlet Allocation models.}
\label{fig:ddpca}
\end{figure}

\subsubsection{Results}

Discriminative sampling is better for both models (Table
\ref{tab:mpcadata}). 
 
\begin{table}[h]
  \caption{Comparison of sampling from the ordinary posterior (jMCMC) and discriminative posterior (dMCMC) for two model families: Mixture of Unigrams (MUM) and mixture of Latent Dirichlet Allocation models (mLDA). The figures are perplexities of the 1000 document test set; there were 100 learning data points. Small perplexity is better; random guessing gives perplexity of 4.}
\label{tab:mpcadata}
\begin{center}
\begin{tabular}{|l|c|c|c|}\hline
Model & dMCMC & jMCMC & Conditional ML\\ \hline
MUM   &  2.56  & 3.98 & 4.84 \\ \hline
mLDA  &  2.36 &  3.92 & 3.14 \\ \hline
\end{tabular}     
\end{center}
\end{table}

\section{Discussion}
We have introduced a principled way of making conditional inference
with a discriminative posterior. Compared to standard joint posterior
density, discriminative posterior results in better inference when the
model family is incorrect, which is usually the case. Compared to
purely discriminative modeling, discriminative posterior is better in
case of small data sets if the model family is at least approximately
correct.  Additionally, we have introduced a justified method for
incorporating missing data into discriminative modeling.

Joint density modeling, discriminative joint density modeling, and
Bayesian regression can be seen as making different assumptions on the
margin model $p(\x\mid\theta)$.  Joint density modeling assumes the
model family to be correct, and hence also the model of $\x$ margin to
be correct. If this assumption holds, the discriminative posterior and
joint density modeling will asymptotically give the same result. On
the other hand, if the assumption does not hold, the discriminative
joint density modeling will asymptotically give better or at least as
good results. Discriminative joint density modeling assumes that the
margin model $p(\x\mid\theta)$ may be incorrect, but the conditional
model $p(c\mid\x,\theta)$, derived from the joint model that includes
the model for the margin, is in itself at least approximately
correct. Then inference is best made with discriminative posterior as
in this paper. Finally, if the model family is completely incorrect
--- or if there is lots of data --- a larger, discriminative model
family and Bayesian regression should be used.

Another approach to the same problem was suggested in \cite{Lasserre06},
where the traditional generative view to discriminative modeling has
been extended by complementing the conditional model for $p(c\mid \x,
\theta)$ with a model for $p(\x\mid \theta')$, to form the joint
density model $p(c,\x\mid \theta,\theta') = p(c\mid \x, \theta)
p(\x\mid \theta')$. That is, a larger model family is postulated with
additional parameters $\theta'$ for modeling the marginal $\x$. The
conditional density model $p(c\mid \x, \theta)$ is derived by Bayes
rule from a formula for the joint density, $p(c,\x\mid \theta)$, and
the model for the marginal $p(\x\mid \theta')$ is obtained by
marginalizing it.

This conceptualization is very useful for semisupervised learning; the
dependency between $\theta$ and $\theta'$ can be tuned by choosing a
suitable prior, which allows balancing between discriminative and
generative modeling. The optimum for semisupervised learning is found
in between the two extremes. The approach of \cite{Lasserre06}
contains our discriminative posterior distribution as a special case
in the limit where the priors are independent, that is,
$p(\theta,\theta')=p(\theta)p(\theta')$ where the parameters $\theta$
and $\theta'$ can be treated independently.

Also \cite{Lasserre06} can be viewed as giving a theoretical
justification for Bayesian discriminative learning, based on
generative modeling. The work introduces a method of extending a fully
discriminative model into a generative model, making discriminative
learning a special case of optimizing the likelihood (that is, the
case where priors separate). Our work starts from different
assumptions.  We assume that the utility functions can be different,
depending on the goal of the modeler.  As we show in this paper, the
requirement of our axiom 5 that the utility should agree with the
posterior will eventually lead us to the proper form of the
posterior. Here we chose conditional inference as the utility,
obtaining a discriminative posterior. If the utility had been joint
modeling of $(c,\x)$, we would have obtained the standard posterior (the
case of no covariates). If the given model is incorrect the different
utilities lead to different learning schemes. From this point of view,
the approach of \cite{Lasserre06} is principled only if
the ``true'' model belongs to the postulated larger model family
$p(c,\x\mid \theta,\theta')$.

As a practical matter, efficient implementation of sampling from the
discriminative posterior will require more work. The sampling schemes
applied in this paper are simple and computationally intensive; there
exist several more advanced methods which should be tested.

\subsection*{Acknowledgments}
This work was supported in part by the PASCAL2 Network of Excellence of
the European Community. This publication only reflects the authors' views.



      
%


\appendix
\section{Proofs}

We use the notation $\r=(c,\x)$ and denote the set of all 
possible observations $\r$
by $R$.

\begin{footnotesize}
For purposes of some parts of the proof, we assume that the set of
possible observations $R$ is finite. This assumption is not
excessively restrictive, since any well-behaving infinite set of
observations and respective probabilistic models can be approximated
with an arbitrary accuracy by discretizing $R$ to sufficiently many
bins and converting the probabilistic models to the corresponding
multinomial distributions over $R$.

{\bf Example:} Assume the observations $r$ are real numbers in compact
interval $[a,b]$ and they are modelled by a well-behaving probability
density $p(r)$, that is, the set of possible observations $R$ is an
infinite set. We can approximate the distribution $p(x)$ by
partitioning the interval $[a,b]$ into $N$ bins
$I(i)=[a+(i-1)(b-a)/N,a+i(b-a)/N]$ of width $(b-a)/N$ each, where
$i\in\{1,\ldots,N\}$, and assigning each bin a multinomial probability
$\theta_{i}=\int_{I(i)}{p(r)dr}$. One possible choice for the family
of models $p(r)$ would be the Gaussian distributions parametrized by
the mean and variance; the parameter space of these Gaussian
distributions would span a 2-dimensional subspace $\overline\Theta$ in
the $N-1$ dimensional parameter space $\Theta$ of the multinomial
distributions.
\end{footnotesize}

\subsection{From exchangeability $F(h(c',{\bf x}'),p_d(\theta\mid
  (c,{\bf x}))) = F(h(c,{\bf x}),p_d(\theta\mid
  (c',{\bf x}')))$ it follows that the posterior is homomorphic to multiplicativity}



\begin{enumerate}
\item[4.] Exchangeability: The value of the posterior is independent of
  the ordering of the observations. That is, the posterior after two
  observations $(c,{\bf x})$ and $(c',{\bf x}')$ is the same
  irrespective of their ordering: 
  $F(h((c',{\bf x}'),\theta),p_d(\theta\mid
  (c,{\bf x})\cup D)) = F(h((c,{\bf x}),\theta),p_d(\theta\mid
  (c',{\bf x}')\cup D))$.
\end{enumerate}

\begin{proof}
  For simplicity let us denote $x=h(c',{\bf x}')$, $y=h(c,{\bf x})$, and
  $z=p(\theta)$. The exchangeability axiom thus reduces to the problem
  of finding a function $F$ such that
\begin{equation}
  F(x,F(y,z))=F(y,F(x,z))~~~,
\end{equation}
where $F(y,z)=p_d(\theta\mid (c,{\bf x}))$. By denoting $F(x,z)=u$ and
$F(y,z)=v$, the equation becomes $F(x,v)=F(y,u)$.

We begin by assuming that function $F$ is differentiable in both its
arguments (in a similar manner to Cox). Differentiating with respect
to $z$, $y$ and $x$ in turn, and writing $F_1(p,q)$ for
$\fracpartial{F(p,q)}{p}$ and $F_2(p,q)$ for
$\fracpartial{F(p,q)}{q}$, we obtain
\begin{eqnarray}
F_2(x,v)\fracpartial{v}{z}&=&F_2(y,u)\fracpartial{u}{z}\label{eq:diff z} \\
F_2(x,v)\fracpartial{v}{y}&=&F_1(y,u) \label{eq:diff y}\\
F_1(x,v)&=&F_2(y,u)\fracpartial{u}{x}~~~. \label{eq:diff x}
\end{eqnarray}
Differentiating equation (\ref{eq:diff z}) wrt. $z$, $x$ and $y$ in turn, we get
\begin{eqnarray}
  F_{22}(x,v)\left(\fracpartial{v}{z}\right)^2+F_2(x,v)\frac{\partial^2 v}{\partial z^2}&=&F_{22}(y,u)\left(\fracpartial{u}{z}\right)^2+F_2(y,u)\frac{\partial^2 u}{\partial z^2} \label{eq:diff zz}\\
  F_{12}(x,v)\fracpartial{v}{z}&=&F_{22}(y,u)\fracpartial{u}{x}\fracpartial{u}{z}+F_2(y,u)\frac{\partial^2 u}{\partial x \partial z} \label{eq:diff xz}\\
  F_{22}(x,v)\fracpartial{v}{y}\fracpartial{v}{z}+F_2(x,v)\frac{\partial^2 v}{\partial z \partial y}&=&F_{12}(y,u)\fracpartial{u}{z} \label{eq:diff yz}~~~,
\end{eqnarray}
and differentiating equation (\ref{eq:diff y}) wrt. $x$ we get
\begin{equation}
F_{12}(x,v)\fracpartial{v}{y}=F_{12}(y,u)\fracpartial{u}{x} \label{eq:diff yx}~~~.
\end{equation}
By solving $F_{12}(x,v)$ from equation (\ref{eq:diff xz}) and
$F_{12}(y,u)$ from equation (\ref{eq:diff yz}) and inserting into
equation (\ref{eq:diff yx}), we get
\begin{eqnarray}
 \frac{F_{22}(y,u)\fracpartial{u}{x}\fracpartial{u}{z}+F_2(y,u)\frac{\partial^2u}{\partial x \partial z}}{\fracpartial{v}{z}}\fracpartial{v}{y}&=&\frac{F_{22}(x,v)\fracpartial{v}{y}\fracpartial{v}{z}+F_2(x,v)\frac{\partial^2v}{\partial z \partial y}}{\fracpartial{u}{z}}\fracpartial{u}{x}\nonumber \\
 &\Leftrightarrow&\nonumber \\
 F_{22}(y,u)\left(\fracpartial{u}{z}\right)^2-F_{22}(x,v)\left(\fracpartial{v}{z}\right)^2&=&\frac{F_2(x,v)\frac{\partial^2v}{\partial z \partial y}\fracpartial{u}{x}\fracpartial{v}{z}-F_2(y,u)\frac{\partial^2u}{\partial x \partial z}\fracpartial{u}{z}\fracpartial{v}{y}}{\fracpartial{u}{x}\fracpartial{v}{y}}~~~.\nonumber \\ \label{eq:intermediate1}
\end{eqnarray}
By inserting equation (\ref{eq:intermediate1}) into equation
(\ref{eq:diff zz}), we get
\begin{eqnarray}
F_2(x,v)\frac{\partial^2 v}{\partial z^2}-F_2(y,u)\frac{\partial^2 u}{\partial z^2}&=&\frac{F_2(x,v)\frac{\partial^2v}{\partial z \partial y}\fracpartial{u}{x}\fracpartial{v}{z}-F_2(y,u)\frac{\partial^2u}{\partial x \partial z}\fracpartial{u}{z}\fracpartial{v}{y}}{\fracpartial{u}{x}\fracpartial{v}{y}} \nonumber \\
\Leftrightarrow 
\frac{F_2(x,v)}{F_2(y,u)}\left(\frac{\partial^2 v}{\partial z^2}\fracpartial{u}{x}\fracpartial{v}{y}-\frac{\partial^2v}{\partial z \partial y}\fracpartial{u}{x}\fracpartial{v}{z} \right)&=&\frac{\partial^2 u}{\partial z^2}\fracpartial{u}{x}\fracpartial{v}{y}-\frac{\partial^2u}{\partial z \partial x}\fracpartial{u}{z}\fracpartial{v}{y}~~~. \label{eq:intermediate2}
\end{eqnarray}
Notice that by equation (\ref{eq:diff z}) we can write
$\frac{F_2(x,v)}{F_2(y,u)}=\fracpartial{u}{z}/\fracpartial{v}{z}$. Inserting
this into equation (\ref{eq:intermediate2}), and dividing by
$\fracpartial{v}{y}\fracpartial{u}{x}\fracpartial{u}{z}$, the equation
simplifies to
\begin{equation}
\frac{\frac{\partial^2 v}{\partial z^2}}{\fracpartial{v}{z}}-\frac{\frac{\partial^2v}{\partial z \partial y}}{\fracpartial{v}{y}}=\frac{\frac{\partial^2 u}{\partial z^2}}{\fracpartial{u}{z}}-\frac{\frac{\partial^2u}{\partial z \partial x}}{\fracpartial{u}{x}}~~~. \nonumber
\end{equation}
The equation can be written also as
\begin{eqnarray}
\fracpartial{}{z} \ln{\fracpartial{v}{z}} - \fracpartial{}{z} \ln \fracpartial{v}{y}&=&\fracpartial{}{z} \ln \fracpartial{u}{z} - \fracpartial{}{z} \ln \fracpartial{u}{x} \nonumber \\
&\Leftrightarrow&\nonumber \\ 
\fracpartial{}{z} \ln \left(\frac{\fracpartial{v}{z}}{\fracpartial{v}{y}}\right)
&=&
\fracpartial{}{z} \ln \left(\frac{\fracpartial{u}{z}}{\fracpartial{u}{x}}\right)~~~.
\end{eqnarray}
Since the left hand side depends on $y,z$ and right hand side depends
on $x,z$, it follows that both must be functions of only
$z$. Furthermore, since the derivative of the logarithm of a function
is a function of $z$, the function itself must be of the form
\begin{equation}
\frac{\fracpartial{v}{z}}{\fracpartial{v}{y}}=\frac{\Phi_1(y)}{\Phi_2(z)}~~~.\label{eq:intermediate3}
\end{equation}
On the other hand, dividing equation (\ref{eq:diff y}) by equation
(\ref{eq:diff x}), we get
\begin{equation}
\frac{F_2(x,v)}{F_1(x,v)}\fracpartial{v}{y}=\left(\frac{F_2(y,u)}{F_1(y,u)}\fracpartial{u}{x}\right)^{-1}~~.
\end{equation}
By equation (\ref{eq:intermediate3}),
$\frac{F_2(x,v)}{F_1(x,v)}=\frac{\Phi_1(x)}{\Phi_2(v)}$. Inserting
this, we get
\begin{eqnarray}
\frac{\Phi_1(x)}{\Phi_2(v)}\fracpartial{v}{y}&=&\left(\frac{\Phi_1(y)}{\Phi_2(u)}\fracpartial{u}{x}\right)^{-1}\nonumber \\
\Leftrightarrow\nonumber \\
\frac{\Phi_1(y)}{\Phi_2(v)}\fracpartial{v}{y}&=&\left(\frac{\Phi_1(x)}{\Phi_2(u)}\fracpartial{u}{x}\right)^{-1}~~~.
\end{eqnarray}
Now, since left hand side depends only on $(y,z)$ and right hand side
on $(x,z)$, each must be a function of $z$ only, that is
$g(z)$. Furthermore, we note that we must have
\begin{equation}
g(z)=\frac{1}{g(z)} \Leftrightarrow g(z)=\pm 1~~.
\end{equation}
Since the condition must be fulfilled for all $x,y$, we must have
\begin{eqnarray}
\frac{\Phi_1(y)}{\Phi_2(v)}\fracpartial{v}{y}=1\nonumber \\
\left(\frac{\Phi_1(y)}{\Phi_2(v)}\fracpartial{v}{y}\right)^{-1}=1\nonumber
\end{eqnarray}
for each $y$ as well. Summing these, we get
\begin{eqnarray}
\frac{\Phi_1(y)}{\Phi_2(v)}\fracpartial{v}{y}+\frac{\Phi_2(v)}{\Phi_1(y)}\frac{1}{\fracpartial{v}{y}}&=&2\nonumber \\
\Leftrightarrow \left(\fracpartial{v}{y} - \frac{\Phi_2(v)}{\Phi_1(y)} \right)^2&=&0~~~. \nonumber
\end{eqnarray}
We can then write
\begin{eqnarray}
\fracpartial{v}{y}&=&\frac{\Phi_2(v)}{\Phi_1(y)}\nonumber \\
\fracpartial{v}{z}&=&\frac{\Phi_2(v)}{\Phi_2(z)}\nonumber~~~.
\end{eqnarray}
Combining these to obtain the differential $d v$ we get
\begin{equation}
\frac{dv}{\Phi_2(v)}=\frac{dy}{\Phi_1(y)}+\frac{dz}{\Phi_2(z)}~~~.
\end{equation}
By denoting $\int \frac{dp}{\Phi_k(p)} = \ln f_k(p)$, we obtain
\begin{equation}
Cf_2(v)=f_1(y)f_2(z)~~~.
\end{equation}
The function $f_1$ can be incorporated into our model, that is $f_1
\circ h \mapsto h$. By inserting $v=F(y,z)$ we get the final form
\begin{equation}
Cf_2(F(y,z))=h(y)f_2(z)~~~.
\end{equation}
\hfill $\Box$
\end{proof}

\subsection{Mapping from $p(c\mid {\bf x} , \theta)$ to $h(\r,\theta)$ is Monotonically Increasing}

\begin{proposition}
From axiom 5 it follows that 
\begin{equation}\label{eq:h_finalA}
\log{h(\r,\theta)} = f_C(\log{p(c|{\bf x} , \theta)})
\end{equation}~where $f_C$ is a monotonically increasing function.
\end{proposition}

\begin{proof} 
Denoting
$\tilde\theta_r = p(\r|\tilde\theta)$ in inequalities (5) and  (6) in the paper
we can write them in the following form
\begin{equation}\label{eq:ax4hA}
\sum_{\r \in R} {\tilde\theta_r \;\log{h(\r,\theta_1)}}\le
\sum_{\r \in R} {\tilde\theta_r \;\log{h(\r,\theta_2)}}
\end{equation}
\centerline{$\Updownarrow$}
\begin{equation}\label{eq:ax4pA}
\sum_{\r \in R} {\tilde\theta_r \; \log{p(c|{\bf x}, \theta_1)}}\le
\sum_{\r \in R} {\tilde\theta_r \;\log{p(c|{\bf x}, \theta_2)}}~~~.
\end{equation}
Consider the points in the parameter space
$\Theta$, where $\tilde\theta_k=1$ and $\tilde\theta_i=0$ for $i \ne k$
(``corner points''). 
In these points the linear combinations vanish and the equivalent inequalities 
 (\ref{eq:ax4hA}) and  (\ref{eq:ax4pA})
become  
\begin{equation}\label{eq:ax4hpnew}
\left \{
\begin{array}{c}
\\
 \log{h(\r_k,\theta_1)}\le
   \log{h(\r_k,\theta_2)}\\
 \\
\Updownarrow \\
 \\
  \log{p(c_k|\x_k, \theta_1)}\le
 \log{p(c_k|\x_k, \theta_2)}\\
\\
\end{array}
\right . ~~~.
\end{equation}

Since the functional form of $f_C$
must be the same regardless of the choice of $\tilde\theta$, 
equivalence (\ref{eq:ax4hpnew}) holds everywhere 
in the parameter space, not just in the corners.

From the equivalence (\ref{eq:ax4hpnew}) (and the symmetry
of the models with respect to re-labeling the data items)
it follows that $h(\r, \theta)$ must be of the form
\begin{equation}\nonumber
\log{h(\r,\theta)} = f_C(\log{p(c|\x, \theta)})\end{equation}
where $f_C$ is a monotonically increasing function. $\Box$
\end{proof} 

\subsection{Mapping is of Form $h(\r,\theta)= \exp(\beta) \; p(c\mid \x, \theta)^A$}

\begin{proposition}
For a continuous increasing function $f_C(t)$ for which
\begin{equation}\nonumber
\log{h(\r,\theta)}=f_C\left( \log p(c|\x,\theta)\right)
\end{equation}
it follows from axiom 5  that
\begin{equation}\nonumber
f_C(t)=At + \beta~~~,
\end{equation}
or, equivalently,
\begin{equation}\label{eq:hlinA}
h(\r,\theta)= \exp(\beta) \; p(c|\x,\theta)^A \quad \text{with} \quad A>0.
\end{equation}
\end{proposition}

\begin{proof} 
Note,that we can decompose $K_{COND}$ as
\begin{equation}
K_{COND}(\tilde\theta,\theta)=S(\tilde\theta)-R(\tilde\theta,\theta)~~~,
\end{equation}
where $S(\tilde\theta)=\sum_{p(|\tilde\theta)}{\log{p(c|\x\tilde\theta)}}$
and $R(\tilde\theta,\theta)=\sum_{p(|\tilde\theta)}{\log{p(c|\x\theta)}}$.
Consider any $\tilde\theta$ and the set of points $\theta$ that
satisfy $R(\tilde\theta,\theta)=t$, where $t$ is some
constant. From the fifth axiom (the equality part) it follows that
there must exist a constant $f_{\tilde\theta}(t)$ that defines the
same set of points $\theta$, defined by
$\sum_{\r\in R}{
  p(\r|\tilde\theta)\log{h(\r,\theta)}}=f_{\tilde\theta}(t)$. 
From the inequality part
of the same axiom it follows that $f_{\tilde\theta}$ is a
monotonically increasing function. Hence,
\begin{equation}
\sum_{\r \in R}{
  p(\r|\tilde\theta)\log{h(\r,\theta)}}=f_{\tilde\theta}\left(
\sum_{\r \in R}{
  p(\r|\tilde\theta)\log{p(c|\x, \theta)}}
\right)~~~.
\label{eq:f}
\end{equation}

On the other hand, from equation (\ref{eq:h_finalA}) we know that we can write
\begin{equation}\label{eq:fcA}
\log{h(\r,\theta)}=f_C\left( \log p(c|\x,\theta)\right)~~~.
\end{equation}

So, equations (\ref{eq:f}) and (\ref{eq:fcA}) lead to

\begin{equation}\label{eq:linear?}
f_{\tilde\theta}\left(\sum_{\r \in R}{
  p(\r|\tilde\theta)\log{p(c|\x, \theta)}}\right)
=  \sum_{\r \in R}{
  p(\r|\tilde\theta)f_C\left(\log{p(c|\x, \theta)}\right)} ~~~.
\end{equation}

If we make a variable change $u_i = \log p(c_i \mid \x_i, \theta)$
and denote $p(\r_i|\tilde\theta) = \tilde\theta_i$
for brevity, equation (\ref{eq:linear?}) becomes

\begin{equation}\label{eq:easy}
f_{\tilde\theta}\left(\sum_{i} {\tilde\theta_i \; u_i } \right)
=  \sum_{i}{ \tilde\theta_i \; f_C\left( u_i\right)} ~~~.
\end{equation}

Not all $u_i$ are independent, however: for each fixed $\x$, one of the 
variables $u_l$ is determined by the other $u_i$'s

\begin{equation}\nonumber
\exp(u_i) = p(c_i \mid \x_i, \theta)~~~,
\end{equation}

and 
\begin{equation}\nonumber
\sum_{fixed \; \x} p(c_i \mid \x, \theta) = 1 \Leftrightarrow \sum_{fixed \; \x} \exp(u_i) = 1~~~.
\end{equation}
So the last $u_l$ for each $\x$ is
\begin{equation}
u_l = \log \left( 1- \sum_{\tiny \begin{array}{c} fixed \; \x\\indep. \; u_m \end{array}} \exp(u_m)\right )~~~,
\end{equation}
where the sum only includes the independent variables $u_m$ for the fixed $\x$.
This way we can make the dependency on each $u_i$ explicit in equation
(\ref{eq:easy}):

\begin{equation}\nonumber
f_{\tilde\theta}\left [\sum_{indep. \; u_j} {\tilde\theta_j \; u_j }
 + \sum_{dependent \; u_l} \tilde\theta_l \;  
\log \left( 1- \sum_{\tiny \begin{array}{c} fixed \; \x\\indep. \; u_m \end{array}} 
\exp(u_m)\right )
 \right]
\end{equation}
\begin{equation}\label{eq:difficult}
= \sum_{indep. \; u_j} {\tilde\theta_j \; f_C(u_j) } +
\sum_{dependent \; u_l}{\tilde\theta_l \; f_C \left (
\log \left( 1- \sum_{\tiny\begin{array}{c} fixed \; \x\\indep. \; u_m \end{array}} \exp(u_m)\right ) \right)}~~~.
\end{equation}

Let us differentiate both sides with respect to a $u_k$:

\begin{eqnarray}\label{eq:difficult2}
& \underbrace{f'_{\tilde\theta}\left ( \sum_{i} {\tilde\theta_i \; u_i } \right )}_{\alpha} \;
\left [ \tilde\theta_k - \underbrace{\frac{\tilde\theta_l}{u_l}}_{c_x} \; \exp(u_k)
 \right]\nonumber\\
= &\tilde\theta_k \; f'_C(u_k) -  \underbrace{\frac{\tilde\theta_l}{u_l}}_{c_x} \;
\underbrace{f'_C(u_l)}_{d_x}\; \exp(u_k)~~~.
\end{eqnarray}

For all such variables $u_k$ that share the same $\x$, we get

\begin{equation}\nonumber
\alpha \; \left [ \tilde\theta_k - c_x\; \exp(u_k) \right ] = 
\tilde\theta_k \; f'_C(u_k) - c_x\;d_x\; \exp(u_k)
\end{equation}
\centerline{$\Updownarrow$}
\begin{equation}\nonumber
\tilde\theta_k \; \exp(-u_k) \; \left (f'_C(u_k) - \alpha \right )
= c_x\;( d_x - \alpha)~~~.
\end{equation}

Since the right-hand side only depends on $\x$, not on individual $u_k$,
the left-hand side must also only depend on $\x$ and the factors depending on
$u_k$ must cancel out.

\begin{equation}\nonumber
f'_C(u_k) - \alpha = B_x \; \frac{\exp(u_k)}{\tilde\theta_k }
\end{equation}
\centerline{$\Updownarrow$}
\begin{equation}\nonumber
\underbrace{f'_{\tilde\theta}\left ( \sum_{i} {\tilde\theta_i \; u_i } 
\right )}_{does \; not\;  depend \; on\;  u_k}
=\underbrace{f'_C(u_k) - B_x \; \frac{\exp(u_k)}{\tilde\theta_k }}_{depends \; on \; u_k}~~~.
\end{equation}
Since the left-hand side depends neither on $u_k$ nor $\x$, both sides must be constant
\[
\implies f'_{\tilde\theta}(t) = A 
\]
\begin{equation}\label{eq:linear}
\implies f_{\tilde\theta}(t) = A\;t + \beta~~~.
\end{equation}

Substituting  (\ref{eq:linear}) into equation
(\ref{eq:easy}) we get
\begin{equation}\nonumber
A \; \left ( \sum_{i} {\tilde\theta_i \; u_i } 
\right ) + \beta = \sum_{i} {\tilde\theta_i \; f_C(u_i) } 
\end{equation}
\centerline{$\Updownarrow$}
\begin{equation}\nonumber
\sum_{i} { \tilde\theta_i \; \left ( A \;u_i - f_C(u_i) \right ) }= - \beta~~~,
\end{equation}
and since this must hold for any parameters $\tilde\theta$, it must also
hold for the corner points:
\[
 A \;u_i - f_C(u_i)  = - \beta 
\]
\begin{equation}
 \implies f_C(t) = A \; t + \beta~~~.
\end{equation}
\hfill $\Box$
\end{proof}

\subsection{Axiom 6 Implies Exponent $A=1$}

\begin{proposition}
From axiom 6 it follows that $A=1$.
\end{proposition}

\begin{proof}
Without axiom 6, the discriminative posterior would be unique up
to a positive constant $A$:
\begin{equation}
p_d(\theta\mid D)\propto p(\theta)\prod_{(c,\x)\in D}{p(c\mid
\x,\theta)^A}.
\end{equation}

Axiom 6 is used to fix this constant to unity by requiring the
discriminative posterior to obey the Bayesian convention for a
fixed $\x$, that is, the discriminative modeling should reduce
to Bayesian joint modeling when there is only one covariate.  
We require
that for a fixed
${\bf{x}}$ and a data set $D_x=\{(c,{\bf{x}}')\in D\mid {\bf{x}}'=
{\bf{x}}\}$ the discriminative posterior matches the
joint posterior $p_j^x$ of a model
$p^x$, where $p^x(c\mid \theta)\equiv p(c\mid \x,\theta)$,
\begin{equation}
p_j^x(\theta\mid D_x)\propto p(\theta) \prod_{c\in D_x}{p^x(c\mid \theta)}
\propto p(\theta)\prod_{(c,\x)\in D_x}{p(c\mid
\x,\theta)},
\end{equation}
Clearly, $A=1$ satisfies the axiom 6, i.e., $p_j^x(\theta\mid D_x)$
equals $p_d(\theta\mid D_x)$ for all $\theta\in\overline\Theta$. If
the proposal would be false, the axiom 6 should be satisfied for some
$A\ne 1$ and for all $\theta$ and data sets. 
In particular, the result should hold for a data set having a single
element, $(c,\x)$. The discriminative posterior would
in this case read
\begin{equation}
p_d(\theta\mid D_x)=\frac{1}{Z_A}p(\theta)p(c\mid\x,\theta)^A,
\end{equation}
where $D_x=\{(c,\x)\}$ and $Z_A$ is a normalization factor, chosen
so that the posterior satisfies
$\int{p_d(\theta\mid D)d\theta}=1$. The joint posterior would, 
on the other hand, read
\begin{equation}
p_j^x(\theta\mid D_x)=\frac{1}{Z_1}p(\theta)p(c\mid\x,\theta) .
\end{equation}
These two posteriors should be equal for all $\theta$:
\begin{equation}
1=\frac{p_d(\theta\mid D_x)}{p_j^x(\theta\mid D_x)}
= \frac{Z_1}{Z_A}p(c\mid\x,\theta)^{A-1}.
\label{eq:ax4ratio}
\end{equation}
Because the normalization factors 
$Z_1$ and $Z_A$ in equation (\ref{eq:ax4ratio}) are
constant in $\theta$, also $p(c\mid\x,\theta)^{A-1}$ must be constant
in $\theta$.  This is possible only if $A=1$ or
$p(c\mid\x,\theta)$ is a trivial function (constant in $\theta$) 
for all $c$.

\hfill $\Box$
\end{proof}

\end{document}